\documentclass[10pt]{article}
\usepackage{graphicx}
\usepackage{epsfig}
\usepackage{amsmath,amsthm, amssymb}
\usepackage{url}
\usepackage[a4paper, lmargin=1.0in, rmargin=1.0in, tmargin=1.0in, bmargin=1.0in]{geometry}

\newcommand{\lan}{{\langle}}
\newcommand{\ran}{{\rangle}}
\newcommand{\N}{\mathcal{N^+}}
\newcommand{\T}{\mathcal{T}}
\newcommand{\p}{\mathbf{Parent_{\Delta}^{min}}}
\newcommand{\R}{\mathbf{R_{\delta_{mut}}}}
\newcommand{\RR}{\mathbf{R^+_{\delta_{mut}}}}
\newcommand{\cO}{\mathcal{O}}
\newcommand{\oP}{\mathcal{P}}
\newcommand{\ancof}{\mathbf{AncestorOf}}
\newcommand{\eu}{\biguplus_{s \in \Sigma_{\T}} E_{s} \times E_{s+1}}
\newtheorem{ob}{Observer Abstraction}
\newtheorem{ax}{Axiom}
\newtheorem{df}{Definition}

\newtheorem{theorem}{Theorem}[section]
\newtheorem{lemma}[theorem]{Lemma}

\newtheorem{cor}[theorem]{Corollary}

\begin{document}


\title{Computational Complexity of Observing Evolution in Artificial-Life Forms}

\author{Janardan Misra \\ 
{\sf janardanmisra@acm.org}}
\date{}
\maketitle

\begin{abstract}
Observations are an essential component of the simulation based studies on artificial-evolutionary systems (AES) by which entities are identified and their behavior is observed to uncover higher-level ``emergent" phenomena. Because of the heterogeneity of AES models and  implicit nature of observations, precise characterization of the observation process, independent of the underlying micro-level reaction semantics of the model, is a difficult problem. Building upon the multiset based algebraic framework to characterize state-space trajectory of AES model simulations, we estimate bounds on computational resource requirements of the process of automatically discovering life-like evolutionary behavior in AES models during simulations. For illustration, we consider the case of Langton's Cellular Automata model and characterize the worst case computational complexity bounds for identifying entity and population level reproduction.     
%
\end{abstract}
\newpage
\section{Background}


Studies on \emph{Artificial Evolutionary Systems} (AES) are recent attempts to complement real-life theories to study the principles underlying the complex phenomena of life without directly working with the real-life organisms. For example, AES studies can complement theoretical biology by uncovering detailed dynamics of evolution where real life
experiments are not possible~\cite{Lenski07,Lenski03}, and by developing generalized formal models for life to determine criterion so that life in any arbitrary model can be observed.

Observations play a fundamental role both in real-life studies as well as in AES research. In case of real-life studies, observations are an integral part of the experimental analysis carried out to uncover the specific dynamics underlying the observed life forms and their properties. In astrobiological studies, though, it is still a challenging problem to detect life in any arbitrary molecular form~\cite{McKay04}. The lack of agreement in interpreting the possible presence of life on Mars with Viking Lander labeled release experiments is a case in example~\cite{levin81,Klein99,miller02}. On the other hand, in case of those AES studies, which focus on the problem of the ``emergence"' of  life-like behavior, in general there is no known method to even decide beforehand the kind of entities, which might emerge demonstrating such non-trivial life-like behavior, without closely observing the simulations. For example, Nehaniv and Dautenhahn~\cite{ND98} have argued that identification of time varying entities is a difficult problem in the context of formal definitions for self-reproduction and add that in absence of observers it is problematic to decide whether an instance of artificial self-replication be treated at all a life-like one. The exceptions occur only when a study starts with built-in design of the entities e.g., well separated programs, which can self replicate, mutate, and evolve according to the design~\cite{ac:Avida,ac:Tierra}. 

Often only informal discussions are presented in AES studies on the mechanisms employed by the researchers to discover the emergent entities and their life-like behavior. These discussions usually remain useful only to the specific models and do not always have the generic perspective. 
Therefore an important aspect where AES studies demand increasing focus is to study observational processes and mechanisms used in AES studies in their own right resulting into a framework for \textit{automated discovery} of life-forms and their dynamics in the simulated environments. Whereas in case of real-life studies such  automatization might not be feasible in general, with AES studies involving mostly digitized universes and their simulations, it is actually desirable to explore by algorithmic means potentially varied possibilities which these simulations hold yet usually require such detailed observations that it may not always be feasible to carry out for human observers alone. Such an automated discovery  of life-forms and the evolving  dynamics may bring much promise in AES studies as compared to what could possibly be achieved only with manually controlled observations. 

An example of such an automated discovery of life forms is discussed by Samaya in~\cite{Samaya98b}. In order to observe the living loops in his CA model, another {\sf ``Observer CA"} system is designed  and embedded within the simulator software. The observer CA is capable of performing the complex image processing operations on the CA configuration given to it as an input by the simulator CA to automatically identify the living loops of different types. 

However, because of its implicit nature and the multitude of AES models, a precise characterization of the observation process is generally a difficult problem. Importantly any sufficiently generic framework for studying the observational processes needs to be defined independent of the low-level micro dynamics or the ``physical laws" of any specific AES model and should permit the study of higher-level observationally ``emergent" phenomena. Initial attempt in this direction appeared in~\cite{Misra06a,Misra07}. The central idea proposed in these studies is a semi-formal characterization of the observation process, which leads to abstractions on the model universe, which are consequently used for establishing the necessary elements and the level of evolutionary behavior in that model. In this paper we extend these ideas further and provide some key results including computational complexity theoretic analysis of the problem of algorithmic realization of the observational processes. 

A multiset based algebraization is proposed in \cite{Misra09a} to achieve an algorithmic characterization of the observational processes for the purpose of entity recognition in artificial-life models. We extend the framework in this work to account for the process of discovery of the life-like evolutionary behavior in AES models. Starting with defining entities and their causal relationships observed during simulations of a model, the extended framework prescribes a series of axioms (conditions) to establish the degree of life-like evolutionary behavior observed in the model. This formal characterization enables to reason about the automated discovery of the life-like entities and their evolutionary behavior, which otherwise is often discovered only manually. 

The framework is analyzed in terms of computational resource requirements of the algorithmic implementation of the observation process and consequent automated discovery of the evolutionary behavior in arbitrary AES models. Building upon the earlier results on the computational complexity bounds for the problem of entity recognition, we further characterize the bounds for an automated discovery of various evolutionary components including entity and population level reproduction with epigenetic developments in the child entities involving mutations, heredity, and natural selection. These computational complexity bounds are established distinguishing further between those AES models which allow entities with overlapping structures to coexist in a state and others which do not.   

The paper is organized as follows: In Section~\ref{chap:framework},  we will formally elaborate the framework followed by computational complexity theoretic analysis of it in Section~\ref{e-complexity}.  Section~\ref{chap:related} presents a discussion of related work, and is followed by concluding remarks in Section~\ref{sec:concluding}, including a discussion on the limitations of the framework and pointers for future research.

\section{The Framework} \label{chap:framework}

\subsection{Preliminaries}


Next we will briefly review the basics of multiset theory and the notions from the theory of computational complexity, which would be used in the formal exposition the framework. 

\subsubsection{Multiset Theory} 

A multiset is a collection of objects, which may contain multiple copies of its elements, e.g., $\{a, a, a, b, b\}$. Formally, a \emph{multiset} $M$ on a set $X$ can be considered as a mapping associating nonnegative integers ($\geq 1$) (representing multiplicity) with each element of $X$, $M: X \rightarrow \N$, where $\mathcal{N^+} = \{1, 2, \ldots\}$. $x \in^r M$ denotes that $x$ appears in $M$ with multiplicity $r$. Some of the operations on the multiset are defined as follows: For multisets $M_1: X \rightarrow \N$ and $\ M_2: Y \rightarrow \N$\\\\
\textit{Multiset Comparison}: Subset: $M_1 \subseteq M_2 \Leftrightarrow [X \subseteq Y] \wedge [\forall x \in X. M_1(x) \leq M_2(x)]$ and Equality: $M_1 = M_2 \Leftrightarrow [X = Y] \wedge [\forall x \in X. M_1(x) = M_2(x)]$\\\\
\textit{Join}: For multisets $M_1, M_2$ on $X$, $M_1 \uplus M_2 = \{(x, n) \mid x \in X \ \textit{and}\  n = M_1(x) + M_2(x)\}$ \\\\
\textit{Intersection}: $M_1 \cap M_2 = \{(z, n) \mid z \in X \cap Y\  \textit{and}\ n = min(M_1(z), M_2(z))\}$\\\\
\textit{Power Set}: For a multiset $M$ on $X$, $\mathcal{P}(M) = \{M'\ on\ X \mid M' \subseteq M\}$\\\\
\textit{Cartesian Product}: 
$M_1 \times M_2: X \times Y \rightarrow \N = \{((x, y), n)\mid x \in X,\ y \in Y, \textit{ and } n = M_1(x) * M_2(x)\}$. So, Cartesian product of a multiset  $M$ on $X$ with classical set $W$ is defined as $M_1 \times W: X \times W \rightarrow \N = \{((x, w), n)\mid x \in X,\ w \in W, \textit{ and } n = M(x)\}$\\\\
\textit{Size}: For a multiset $M$ on $X$, its size $|M| = \Sigma_{x \in X}M(x)$\\\\
\textit{Functions}: The mapping based definition of a multiset given above does not directly support the notion of a function on a multiset. Recall that for classical sets, a function is a special subset of the Cartesian product of the domain and the range sets such that each element of the domain set is associated with only one of the elements of the range set. This conditions when applied over the multisets defined as above cannot be easily generalized since for that we need to distinguish all the multiple copies of the individual elements. Therefore, we restrict our attention to defining a special class of functions, which will be used in the paper. The functions we consider are those which restrict the range to a classical set. Let us consider such a function $f: M \rightarrow W$, where $M$ is a multiset on $X$ and $W$ is the range set. $f$ satisfies the following constraint:\\
$$\forall x \in X.\exists W_x \subseteq W\ s.t.\ [|W_x| = M(x)]\ \wedge \ [\forall w \in W_x\ (x, w) \in f]$$ Informally, this constraint demands all the multiple copies of each $x \in X$ should be associated with different elements of $W$ by the function $f$. Notice that such a function itself is a classical set - this is an important property which will be used later. Consider for example a multiset $M = \{(a, 3), (b,2)\}$ as the domain and the set $W = \{\alpha, \beta, \delta, \lambda, \nu\}$ as the range of the function $f: M \rightarrow W$ defined as $\{(a, \alpha), (a, \beta),(a, \delta)$, $(b,\alpha), (b,\nu)\}$, which satisfies the above constraint with $W_a = \{\alpha, \beta, \delta\}$ and $W_b = \{\alpha, \nu\}$. A special subclass of these functions called \textit{\sf `one-to-one' or `injective'} functions further restrict that no two ordered pairs can have identical second elements in the pairs: If $f$ is `one-to-one' then for all $(x_1, w_1), (x_2, w_2) \in f$ we have $[w_1 = w_2] \Rightarrow [x_1 = x_2]$. For example, a one-to-one function $g: M \rightarrow W$ may be defined as $\{(a, \alpha), (a, \beta),(a, \lambda), (b,\nu), (b,\delta)\}$. \\\\
For further details on multiset theory, reader is referred to~\cite{Syropoulos01mathematicsof,ac:Wayne89,multiset_overview}.

\subsubsection{Computational Complexity}\label{pre-complexity}

The field of computational complexity studies issues related to the computational resources (e.g., CPU time, memory) required for the execution of algorithms and the inherent difficulty in designing algorithms for specific problems~\cite{complexity,algorithms}. For example, {\it time complexity} of a problem is the (minimum) number of time steps that would be required by any algorithm to solve an instance of the problem as a function of the size of the input. 
When computational complexity for a specific problem is considered, necessary resource requirements for any algorithm solving that problem are measured indicating the inherent computational difficulty of that problem. On the other hand, when a specific algorithm for a problem is considered, computational complexity measures are only limited to that algorithm. 

Computational complexity measures are in general specified without explicitly considering the details of the actual underlying machine architecture and rather by assuming some abstract representation of it, e.g., Turing Machine model and RAM (Random Access Machine) model. These abstract models are known as the models of computation and computational complexity measures are specified only in reference to a model of computation. We will consider the RAM model as the underlying model of computation for the computational complexity theoretic discussions to be presented in this paper. RAM model is an example of the traditional von Neumann architecture of a sequential machine which can execute sequential computer programs and is computationally equivalent to the Turing machine model. 

Computational complexity theory often classifies problems into various complexity classes based upon their inherent resource requirements. A {\it computational complexity class} w.r.t. a specific computational resource is the set of all those problems which minimally require at least the amount of resource specified for that class. For example, the time complexity class {\it NP} is the set of all those problems for which a solution can be verified in polynomial number of time steps on a deterministic Turing machine model. Many important problems fall in this class including the Boolean formula satisfiability problem~\cite{complexity}, the protein folding problem~\cite{proteinfolding}, clustering problem, timetable design problem, staff scheduling problem etc. (see \cite{book:grey_npcpmplete} or the wiki page \cite{wiki:list_np_complete} for a detailed list of NP Complete problems.) An important technique used for proving that a problem $\mathcal{P}$ is hard w.r.t. a complexity class $\mathbf{H}$ is known as the {\it reduction method}, in which a relatively simpler translation (as compared to the computational requirements of the class $\mathbf{H}$) is defined, which reduces an input instance of some problem $\mathcal{P'}$, which is already known to be belonging to the class $\mathbf{H}$, into an instance of the problem $\mathcal{P}$.  

Asymptotic order notation known as $\cO(.)$ (big Oh) is often used to measure the bounds on computational complexity for algorithms and problems~\cite{algorithms}. If $f(n) = \cO(g(n))$, then $f$ is said to be upper bounded by $g$ for for all the positive values of the input of size $n$ after certain point. Formally, $$f(n) = \cO(g(n)) \Leftrightarrow \exists n_0 > 0 . \exists k > 0\ \textit{s.t.}\ \forall n > n_0 . |f(n)| < k|g(n|$$ Informally, it means, $f$ grows no faster than $g$. Also we have the following useful asymptotic properties~\cite[page 80]{levitin}: If $f_1(n) = \cO(g_1(n))$ and $f_2(n) = \cO(g_2(n))$,   
\begin{eqnarray}
f_1(n) + f_2(n) &=& \cO(g_1(n))+\cO(g_2(n)) = \cO(\max\{g_1(n), g_2(n)\}),\\ 
f_1(n) * f_2(n) &=& \cO(g_1(n))*\cO(g_2(n)) = \cO(g_1(n)*g_2(n))
\end{eqnarray}
A specific type of computational complexity analysis often considered is the  {\it worst case analysis}, which analyses the resource requirements for solving the worst possible input instance of the problem. Such analysis sets an upper bound on the resource requirements for any input instance of the problem. For example, for the problem of searching for a given item in an input list of size $n$, the worst case time complexity would be $\cO(n)$ for the case where the item to be searched is not present in the list. Since with arbitrary amount of resources any input instance for a computable problem can be solved, worst case analysis defines the lowest possible bounds for solving the worst case input instance. Sometimes, however, it is not possible to give exact (tightest) bounds for a problem under consideration, in that case, some relatively relaxed upper bound is given instead. 
The worst case analysis of a specific algorithm solving the problem provides one such {\it upper bound} on the worst case computational complexity of that problem itself since it demonstrates that those computational resources would be sufficient to solve any instance of the problem, though another better algorithm may exist requiring less resources.  

\subsection{The Formal Structure of the Framework}
\label{chap:formalstruct}

To illustrate the framework, we will use the examples of Cellular automata (CA) based Langton Loops~\cite{ac:Lan84} and $\lambda$ calculus based Algorithmic chemistry\cite{ac:Fon92}. Brief background on these models is given next.

John von Neumann defined CA~\cite{Neumann66} to explain the generic logic of self reproduction in mechanical terms. His synchronous CA model was a two dimensional grid divided into cells, where each cell would change in parallel its state based upon the states of its neighborhood cells, its own state and its transition rule. For this CA model, von Neumann defined a virtual configuration space where he demonstrated analytically that there exists some universal replicator configuration which could replicate other configurations as well as  itself. Though universal replicators are not found in nature and such self replicator was extremely large in its size, the underlying logic of treating states of cells both as `data' as well as `instruction' was a fundamental contribution of this model since it was also discovered later in case of real-life where DNA sequences specify both transcription as well as translation for their own replication in a cell. Another strength of von Neumann's formulation was its ability to give rise to unlimited variety of self replicators \cite{McMullin00a,McMullin00b}. Over the years this model was simplified and reduced in size considerably \cite[81-105]{Codd68}.

Finally Langton introduced loop like self replicating structures in~\cite{ac:Lan84}, which retained the `transcription - translation' property of von Neumann's model excluding the capability of universal replication and symbolic computation. Langton's original self-replicating structure is a 86-cell loop constructed in two-dimensional, 8-state, 5-neighborhood cellular space consisting of a string of core cells in state $1$, surrounded by sheath cells in state $2$. These loops have since then, been extended into several interesting directions including evolving Evoloops in \cite{Samaya98b}. 

Algorithmic Chemistry (AlChemy) was introduced in \cite{ac:Fon92} and further discussed in \cite{ac:FB94arrival,ac:FWB94,FB94twice,ac:FB96}. The main focus of AlChemy is to study the principles behind the emergence of molecular organizations through an approximate abstraction of real chemistry as $\lambda$ calculus with finite reductions. Starting with a random population of $\lambda$ terms (molecules), using different filtering conditions on reactions, authors describe the emergence of different kinds of organizations: \emph{Level $0$} organization consisting of a set of self copying $\lambda$ term(s) and hypercycles with mutually copying $\lambda$ terms, \emph{Level $1$} \emph{self maintaining} organizations consisting of $\lambda$ terms such that every term is effectively produced as a result of reaction between some other terms in the same organization and lastly \emph{Level $2$} meta-organization consisting of two or more Level $1$ sub organizations such that molecules migrate between these self maintaining sub-organizations. Authors also discuss an algebraic characterization of Level $1$ and Level $2$ organizations without referring to the underlying syntactical structure of the $\lambda$ terms (molecules) or the micro dynamics (reduction semantics and filtering conditions) governing the output of reactions. Further details on these models can be found in the above references.

In the ensuing discussion, we will use ``AES model" and ``model", ``Observation process" and ``Observer" interchangeably to add convenience in presentation. Similarly ``real-life" is used in the paper to refer to organic life on earth in contrast with the ``artificial-life". Also, \emph{Observer Abstractions} will refer to specific observations and corresponding abstractions made upon the AES model during its simulations. \emph{Axioms} are used to specify conditions which need to be satisfied in order to draw valid inferences e.g., recognition of entities and their causal relationships. The aim is to define these Axioms such that only valid claims for the presence of life-like phenomena in a model can be entertained. Auxiliary formal structures are also used in the intermediate stages of analysis. E.g., distance measure for determining dissimilarity between entities with respect to their specific characteristics.

\subsubsection{Observation Process and the Model Universe}\label{sec:obproc}

We first briefly summarize the framework presented in~\cite{Misra09a} for the purpose of entity recognition in artificial-life models. Next, we will extend it for the purpose of observing evolutionary components.\\\\
\textbf{Observation Process.} {\em $\Gamma \mapsto_{Obj} \Pi$: An observation process $Obj$ is defined as a computable
transformation from the underlying model structure $\Gamma = (\Sigma_0, \Sigma, \delta)$ to observer abstractions $\Pi = (Abs_{ind},\ Abs_{dep})$,  where $Abs_{ind} = (E,\ \Upsilon,\ D,\ \delta_{mut},\ C)$ is the set of process independent abstractions and $Abs_{dep}$ is the set of process dependent abstractions. }\\\\
\textbf{States}. {\em $\Sigma$: set of observed states\footnote{We will use $s, s', s_1, s_2, \ldots$ to denote individual states.} of the model across simulations.}\\

In general a state could be considered as a collection of all observable atomic structural elements and their (observationally relevant) characteristics in the model at any given time point during the simulation. A multiset xan be used to represent state of a model at any instance during its simulation. \\

\textbf{Observed Run}. {\em $\mathcal{T}$: set of observed sequence of states ordered with respect to the temporal progression of the model during its simulation. Each such sequence represents one \emph{observed run} of the model. A sequence of states is formally represented as a partial mapping $\T : \Sigma \rightarrow N$, where $N$ is the
set of non negative integers acting as indexes for the states in the sequence.}\\


Let $\Sigma_{\T}$ denote the set of states appearing in a specific run $\T$. An important consideration while defining a run $\T$ is to decide the temporal granularity. The granularity would determine which states an observation process is observing (or is able to observe) and that would in turn determine to what extent life-like behavior could be potentially observed by the observation process.


$\Sigma$ and a set of observed runs $\T{_1}, \T{_2}, \ldots$ thus define the observed dynamic structure of the model as a state machine $\Gamma = (\Sigma_0, \Sigma, \delta)$ with respect to a given observation process. Where 
\begin{itemize}
\item $\Sigma_0 \subseteq \Sigma$ 
is the set of starting states, i.e., $\forall s \in \Sigma_0 . \exists \T_i \ 
\textit{such that}\ \T_i(s) = 0$. 
\item $\delta \subseteq \Sigma \times \Sigma$ is the 
transition relation between states, i.e., $(s, s') \in \delta \Leftrightarrow 
\exists \T_i \ \textit{such that}\ \T_i(s') = \T_i(s) + 1$. 
\end{itemize}

\subsubsection{Entities and Their Characteristics}
\label{sec:observe}

\begin{ob}[\textbf{Entity Set}]
$E_{s}$: Multiset of entities observed and uniquely identified by the observer in a state $s$ of the model for a given run $\T$. $$E_{\T}  = \biguplus_{s \in \Sigma_{\T}} E_s$$ is the multiset of entities observed and uniquely identified by the observer across the states in the given run $\T$.
\end{ob}

The criterion to select the set of uniquely identifiable entities in a given state of the AES model is entirely dependent on the observation process as specified by the AES researcher. 

As an example, consider the case of two dimensional CA lattice based model. A cell in the lattice is represented as $\lan(x, y), i\ran$, where $(x, y)$ is the coordinate and $i \in [0..7]$ is the state of the cell. When a cell is in state $0$, it is also known as a \emph{quiescent} cell. Let $Cell$ denote the set of all cells. For a given cell $c := \lan(x, y), i\ran$, let $co_x(c) = x$ and $co_y(c) = y$, which can be extended to the set of cells: $\forall Z \subseteq Cell$, $co_x^+(Z) = \bigcup_{c \in Z} co_x(c)$, $co_y^+(Z) = \bigcup_{c \in Z} co_y(c)$. Also let $Neigh: Cell \rightarrow \oP(Cell)$ output the coordinate wise non quiescent cells in the surrounding neighborhood of a cell. 
In terms of these, each entity in a state can be characterized by two values - $[z, pivot(z)]$ - the connected set of non-quiescent cells $z \subseteq Cell$ and an associated \emph{pivot}. The (function) $pivot(z)$ gives the coordinates for a cell uniquely associated with an entity in CA lattice in a particular state. Pivot works as an state invariant for each such identified entity (loop).  

In general an observation process may recognize an entity in a state as a subset of the observable atomic structures, that is, (structure of) an entity can be defined as a subset of the state itself. Formally, $$E_{s} \subseteq \oP(s)$$ In case of two entities being identical i.e., consisting of identical multisets of atomic structures, additional tagging may also be required. Notice that such a subset based structural identification of entities also allows entities with overlapping sets of atomic elements. In such scenarios also application of additional tagging is essential. Tagging, in general, can be defined as a total function: $$Tagging: E_{\T} \mapsto \mathit{I_{Tag}}$$ 
Where $\mathit{I_{Tag}}$ is the set of unique tags to be associated with the entities. In the following discussion, we would implicitly assume that whenever necessary, identified entities in each state are also tagged such that with tagging, each identified entity would naturally be distinguishable from others. If two entities are considered different, that would mean either the these entities have distinct structures or they are given different tags by the observation process. 

Also we have the following two axioms imposing consistency requirements on the entity identification. First axiom states that every entity in a state is uniquely identified. Second axiom states that the set of entities identified in identical states should be the same. 

\begin{ax}[\textbf{Axiom of Unique Identification of Entities}] An entity must be uniquely identified in a given observed run $\T$. Formally, {\sf Tagging is a one-to-one function on states}, that is, $\forall s \in \Sigma_{\T}.\ \forall e, e' \in E_{s}.\ Tagging(e) = Tagging(e') \Rightarrow e = e'$.
\end{ax}

\begin{ax}[\textbf{Axiom of Unique Identification in States}] Entity set must be uniquely identified for each state in a given observed run $\T$. Formally, $$\forall s, s' \in \Sigma_{\T}:\ s = s' \Rightarrow E_s = E_{s'}$$\end{ax}
This axiom may also be considered as the \textit{soundness axiom} for entity recognition.  Finally we have the following important axiom of non-ignorance, which 
can be considered as a complementary \textit{completeness axiom} for entity recognition. 

\begin{ax}[\textbf{Axiom of non-Ignorance}] Let $E_{s}$ and $E_{s'}$ denote 
the multisets of entities observed in states $s, s' \in \Sigma_{\T}$ respectively. Then $\forall s, s' \in \Sigma_{\T}:$
\begin{eqnarray*}
  \mathbf{s \subset s'} & \mathbf{\Rightarrow} & \mathbf{E_s\setminus Y = E_{in}}; \ \textit{where}\\
  E_{in} &=& \{e' \in E_{s'} \mid e' \subset s \};\\
  E_{out} &=& \{e' \in E_{s'} \mid e' \cap s = \emptyset \};\\
  E_{overlap} &=& E_{s'}\setminus (E_{in} \uplus E_{out}) ;\\
  Y &=& \{e \in E_{s} \mid (\exists e' \in E_{overlap}). [e \cap e' \neq \emptyset]\};
\end{eqnarray*}
\end{ax}
In other words, this axiom forbids that an observer omit identification of an entity in a state $s$ but  in a different state $s'$ identifies it as consisting of the same atomic elements which were also available in  $s$. 

Having defined the sequence of states with temporal ordering and the entities identified by their tags, we will now proceed to discuss how an observer might define the detailed observable characteristics for such entities. Using these characteristics it can infer relevant relationships among entities e.g., descendant relationship, heredity, and variation. To this aim, we will define `character space' as a set of values for the observed characteristics. These values might be purely symbolic without any relative ordering or can be ordered using suitable ordering relation.

\begin{ob}[\textbf{Character Space}]
The observer should define the set of all possible mutually independent (or orthogonal) and {\sf measurable} characteristics for possible entities in the model as a multi dimensional character space $\Upsilon = \mathit{Char}_1 \times \mathit{Char}_2 \times \ldots \times \mathit{Char}_n$, where each of $\mathit{Char}_i$ is the set of values for $i^{th}$ characteristic. Each of $\mathit{Char}_i$ make one orthogonal dimension in the space $\Upsilon$. 
\end{ob}
Corresponding to each entity $e \in E_{\T}$ there is a point in $\Upsilon$, say $(v_1, v_2, \ldots v_n)$, where $v_i \in \mathit{Char}_i$. For a vector $x = (a_1, a_2, \ldots, a_r), i^{th}$ element ($a_i$) will be denoted as $x[i]$. For some of the characteristics observer might define a `partial ordering' ($\leq_i$ for $char_i \in \Upsilon$), which can be used to compare values for those characteristics. The absence of a characteristics in an entity is represented by a special zero element $0_{char_i}$ such that if $\mathit{char}_i$ is (partially) ordered then $\forall v \in
\mathit{char}_i$. $0_{char_i} \leq_i v$.

Notice that, observable characteristics need not to be limited to syntactic level  or \emph{structural properties} and can also include semantic properties - \emph{observable patterns of behaviors} - though semantic properties are much more difficult to observe and measure since they require abstracting the patterns of reactions over a range of states and also semantic equivalence between entities might be computationally undecidable, in general, e.g., program equivalence~\cite[Ch. 8]{hopcroft79}. 

In case of Langton's CA model, an obvious characterization for $\Upsilon$ is a two dimensional space $\mathit{Char}_1 \times \mathit{Char}_2$ with $\mathit{Char}_1$ being the set of all non-quiescent connected set of cells and $\mathit{Char}_2$ being the set of corresponding pivots.
%
We will next extend the framework to infer various components of the life-like evolutionary behavior. 

\subsubsection{Distance Measure}
\label{sec:distance}

A ``dissimilarity measure" ($D$) defines the ``observable differences" ($\mathbf{Diff}$) between the characteristics of the entities in a population. The distance measure defined below can be used by the observer to distribute entities into separate clusters such that entities in the same cluster are sufficiently similar while entities from different clusters are distinguishably different in their characteristics. Exact definition of distance function is also model dependent.

\begin{ob}[\textbf{Distance Measure}] An observer defines a decidable clustering distance measure $D: E_{\T} \times E_{\T} \rightarrow \mathbf{Diff}$, where $\mathbf{Diff}$ is the set of values to characterize the observable ``differences" between entities in $E$.
\end{ob}

Examples include the Hamming distance to define distance between genomic strings in the Eigen's model of molecular evolution \cite{ps01}, set of points where two computable functions differ in their function graphs, or the set of instructions where two programs may differ. One of the known criterion to define the concept of species is ``phenotype similarity" \cite{Ridley96}, which can also be seen as another example for distance measure.

In case of Langton's CA model, $D: E_{\T} \times E_{\T} \rightarrow \{0, 1\} \times \{0, 1\}$ is defined such that $\forall e, e' \in E_{\T}\ .\ D(e, e') = [d_g, d_p]$ where $d_g$ is $0$ if both entities have the same number of cells arranged identically or else it is $1$ and $d_p$ is $0$ when the pivots for both the entities are the same and
$1$ otherwise.

\subsubsection{Observable Limits on Mutational Changes}
\label{sec:mut-limits}

For most of the non trivial AES models entities may also change (mutate) over the course of their interaction with the environment (or other entities.) These changes in the structure as well as the characteristics of the entities might in turn make it difficult for an observer to recognize the  entities across states. Therefore it is essential to specify 
the limits under which an observer can recognize entities across states even in the presence of such mutational changes. This is an inherent limiting property on the part of the observer and could vary among observers. Based upon the limit referred here as $\delta_{mut}$, an observer can establish whether two entities in different successive 
states are indeed the same with differences owning to mutations or not.  The smaller the limit, the harder it will be for an observer to keep recognizing entities across states and a mutated entity may get classified as a new one. As entities are observed in more and more refined levels of details, their apparent similarities melt away and differences become sharply noticeable. Another alternative method which could be used for establishing the persistence of entities across states is the identification of temporally invariant properties of the entities of interest. In the current framework such temporally invariant  properties of entities could also be defined as a separate characteristic in $\Upsilon$. 

\begin{ob}[\textbf{Mutation Bound}] Based upon the choice of clustering distance measure $D$, the observer selects some suitable $\delta_{mut} \in \mathbf{Diff}$, which will be used to bound mutational changes for proper recognition. $\delta_{mut}$ is a vector such that each element specifies an observer-defined threshold on the recognizable mutational changes for corresponding characteristic.
\end{ob}

It is important to note that the choice of $\delta_{mut}$ critically affects further inferences. For example, a choice of very large values would result in the lack of identification of variability in characteristics among entities. This in turn might make an observer to recognize several (otherwise non identical) entities as identical and thus the dynamics of model which could have been possible to uncover with fine grained identification of the differences would not be possible. On the other hand if an observer decides to select very small values for $\delta_{mut}$ then it cannot recognize persistence of an entity across states under changes. Therefore it is important to define the bounds optimally. Next, we define \emph{Recognition relation} to establish the persistence of entities even in the presence of mutational changes:

In case of Langton's CA model, let us select $\delta_{mut} = [1, 0]$, which means the observer can recognize an entity in future states even with mutations (changes in the states, number, or the arrangement of cells comprising the entity) provided that the pivot remains the same. 

\begin{df}[\textbf{Recognition Relation}]
The observer establishes recognition of entities across states of the model with (or without) mutations by defining the function $\mathbf{R_{\delta_{mut}}}$: $E_{\T} \rightsquigarrow E_{\T}$, which is a partial function and satisfies the following axioms: 
\end{df}

\begin{ax}
$\forall e, e' \in E_{\T}\ .\  \mathbf{R_{\delta_{mut}}}(e) = e'
\Rightarrow\ \textit{{\sf if}}\  \exists s \in \Sigma_{\T}\ s.t.\ e \in E_s \ \textit{{\sf then}}\  e' \in E_{s+1}$.
\end{ax}

Informally, the axiom states that entities to be recognized as the same have be observed in successive states. Note that $\mathbf{R_{\delta_{mut}}}$ is anti symmetric (and therefore partial) to ensure that entities are recognized based upon the temporal progression of the model and not in any other arbitrary order.

\begin{ax} $\mathbf{R_{\delta_{mut}}}$ is an {\sf injective} function, that is, $\forall s \in E_{\T}.\ \forall e, e' \in E_{s}.$ $\mathbf{R_{\delta_{mut}}}$$(e) =$ $\mathbf{R_{\delta_{mut}}}$$(e')$ $\Rightarrow e = e'$
\end{ax}

Informally, the axiom states that no two different entities in one state can be recognized as the same in the next state. 

\begin{ax}
$\forall e, e' \in E_{\T}$. $\forall \mathit{char}_i \in \Upsilon$. $\mathbf{R_{\delta_{mut}}}$$(e) = e' \Rightarrow
0_{\mathit{diff}_i} \preceq_i D(e, e')[i] \preceq_i \delta_{mut}[i]$
\end{ax}

Informally $\mathbf{R_{\delta_{mut}}}$$(e)$ is that $e' \in E_{\T}$, which is recognized in the next state by the observer as $e$ with possible mutations bounded by $\delta_{mut}$. In other words if entity $e$ mutates and changes in the next state and identified as $e'$, then observer might be able to recognize $e$ and $e'$ as the same if these changes (between $e$ and $e'$) are bounded by $\delta_{mut}$.

In case of our example of Langton's CA model, a recognition relation satisfying above axioms would  imply that two entities in consecutive states are recognized same only if they have the same pivots. This also means the observer can recognize an entity even with change in the number, state, and geometrical arrangement in the cells across states provided that entity does not shift in CA lattice altogether (which would result in the change of the pivot.) Axiom $6$, which states that $\mathbf{R_{\delta_{mut}}}$ is an injective function would also hold because no two entities in the same state share the same pivot. This is because pivot as defined before is connected to all other cells of the entity and all the non-quiescent cells which are connected in any state are taken together as one entity. Thus two different entities in the same state always consist of cells such that cells in one entity are not connected with the cells of second entity, and hence always have different pivots.

\subsubsection{Observed Causality}

In order to infer meaningful relationship between entities, which could later be used as a basis  for inferring interesting macro level properties in the model, an observer needs to identify ``causal'' relationships among entities independent of the underlying `physical laws' of the model.
  
\begin{ob}[\textbf{Causality}]
$\displaystyle C \subseteq \eu$. $C$ establishes the observed causality among the entities appearing in the successive states of a run $\T$.
\end{ob}

Notice that in order to establish causal relation between entities, observer need not necessarily know the underlying reaction semantics or the micro level dynamics of the model. Only requirement is that the observer's claimed causality conforms with the stated axioms. Since causality is largely a observer and model dependent property, we will illustrate it further in Section~\ref{reproformalized} by defining an additional axiom of reproductive causality, which will be used in turn to infer reproductive relationships among entities.

For example, in case of Langton's CA model, the relation $C$ between entities in consecutive states is defined as follows: $C \subseteq E_{\T} \times E_{\T}$
such that $\forall\ e, e' \in E_{\T}$ where $e = [z_e, \mathit{pivot}(z_e)]$ and $e' =[z_{e'}, \mathit{pivot}(z_{e'})]$ we require
\[
(e, e') \in C \Leftrightarrow \left\{ \begin{array}{ll} 1.\, & co_x^+(z_e) \supset co_x^+(z_{e'})\\
2.\, & co_y^+(z_e) \supset co_y^+(z_{e'})\\
3.\, &  \mathit{pivot}(z_e) \neq \mathit{pivot}(z_{e'})\\
4. & \exists s \in \Sigma_{\T}\ s.t.\ e \in E_s \wedge e' \in E_{s+1}
\end{array} \right.
\]
Intuitively what we demand with above definition of causal relation $C$ is that (child) entity $e'$ was part of the (parent) entity $e$ and at certain stage it ``breaks off'' from the (parent) entity $e$, as can be seen in Figure~\ref{ftr1} at time step $127$.

\subsection{Observing Evolution}
\label{chap:evocomponents}

Having defined the observation process as a computable transformation $\mapsto_{Obj}$ from the underlying sequence of observed states of the model $\Gamma$  to the set of components $Abs_{ind}$ involving entities and their observable characteristics with measurable differences as well as observable limits on such differences, we will now consider a specific observation process for observing evolutionary behavior in the AES studies. The following discussion will define components in $Abs_{dep}$ for observing the fundamental evolutionary components: reproduction with mutations and epigenetic developments, heredity, and natural selection. $Abs_{dep}$ can also be defined considering other criterion of recognizing life-like phenomena, for example, metabolism~\cite{ac:BFF92}, complexity~\cite{ac:adami00}, self organization~\cite{ac:Kau93}, autonomy and autopoisis~\cite{Zeleny81}.

\subsubsection{Reproduction}
\label{reproformalized}

Reproduction is one of the fundamental components of evolution. Through reproduction, entities pass on their characteristics to the next generation and increase the population size. Reproduction is possibly the only way by which (abstract) entity features can persist across generations in case of those AES models, where entities do not persist forever. In an observed framework, the way an observer can establish reproduction is by providing observed evidence for it. This is done by defining causal descendance
relationships among the entities across states, whereby parent and the child entities are recognized by the observer as being sufficiently similar and ``causally'' connected across the states. Formally, we add a new Axiom for the causal relation defined before:

\begin{ax}[\textbf{Reproductive Causality}]
$\forall s \in E_{\T}.\ \forall e \in E_s, e' \in E_{s+1}.\  (e, e') \in C \Rightarrow [\not\!\exists e'' \in E_s.\  \mathbf{R_{\delta_{mut}}}$ $(e'') = e']$
\end{ax}

Informally, the axiom of reproductive causality states that if an entity $e$ in state $s$ is causally connected to entity $e'$ in the next state, then there must not be any other entity $e''$ in state $s$, which is also recognized by the observer as $e'$ in the next state. This is to ensure that mutations are not confused by the observer with reproductions. In essence, this formulation of causality is an abstract specification which demands observers to identify the entities which have been observed to be causal sources for the appearance of a new entity. Only then proper descendance relation for the new entity can be established.

\begin{lemma} Causal relation $C$ defined for the Langton's CA model satisfies the Causality Axioms $8$ and Reproductive Causality Axiom $9$. \end{lemma}
\begin{proof}
Condition $4$ insures that $e$ and $e'$ are observed only in the consecutive states as demanded by the axiom $8$. To establish that 
$e'$ is not the result of mutations in some other entity $e''$ observed in past (i.e., $[\exists s \in \Sigma_{\T}\ s.t.\ e'', e \in E_s]\wedge 
[\mathbf{R_{\delta_{mut}}}(e'') = e']$) we note that because of the definition of $\mathbf{R_{\delta_{mut}}}$, $e''$ and $e'$ would otherwise have the same pivots, which means pivot of $e''$ will be included in the set of cells in $e$ (since $[co_x^+(z_e) \supset co_x^+(z_{e'})] \wedge [co_y^+(z_e) \supset co_y^+(z_{e'})]$), which is not possible because $e$ and $e''$ being different entities in the same state cannot have cells in common including pivot as argued above in the proof of previous lemma.
\end{proof}

\subsubsection*{Reproductive Mutations}\label{mutformalized}

For evolution to be effective there should be observable differences between the child and the parent entities arising out of reproductive processes. These changes in the characteristics of the entities may or may not be inheritable based upon the design of the model and the simulation instance.

\textbf{{\sf Observable Limits on Reproductive Mutational Changes:}}
\label{sec:rep-limits}
Similar to the above discussed $\delta_{mut}$,  it is important to specify the limits under which an observer can identify whether an entity is an descendant of another entity even though they might not be identical. This necessitates us to introduce another bound on observable reproductive mutations as $\delta_{rep\_mut}$. This limit on observable reproductive mutations is indeed crucial while working with models where epigenetic development in the entities can be observed \cite{Mah97}. This is because in such models
including examples from real life, the ``child" entity and the ``parent" entities do not resemble with each other at the beginning and observer has to wait until whole epigenetic
developmental process gets unfolded and then compare the entities for similarities in their characteristics.  $\delta_{rep\_mut}$ assists an observer to establish whether a particular entity could be treated as a ``descendant" of another entity or not.

Another reason for introducing the limit $\delta_{rep\_mut}$ is that from the view point of an high level observation process not recording every micro level details, it is quite essential to distinguish between parent entities and other secondary entities involved in the reproductive process. Consider, for example, a model where entity $A$ reproduces according to reaction $A + B  \rightarrow 2A' +G$, where $A'$ is mutant child entity of $A$, which can be determined by an observation process only when it
can establish that $A$ and $A'$ are sufficiently similar with respect to their characteristics, while $A'$ and $B$ are not.

\begin{ob}[\textbf{Reproductive Mutation Bound}] Based upon the choice of clustering distance measure $D$, the observer selects some suitable $\delta_{rep\_mut} \in \mathbf{Diff}$, which will be used to bound reproductive mutational changes for proper recognition.  $\delta_{rep\_mut}$ is a vector such that each element specifies an
observer-defined threshold on the recognizable mutational changes for corresponding characteristics.
\end{ob}

In case of Langton's CA model, let $\delta_{rep\_mut}= [0,1]$, which means for reproduction observer strictly demands identical geometrical structure of the parent and child entities, though they may have different pivots - this is essential to capture exact replication of the loops.

It is important to note that the choice of $\delta_{rep\_mut}$ critically affects further inferences. For example, small values for $\delta_{rep\_mut}$ might make it harder to establish reproductive relationships among entities and  for such an observer every new entity would seem to be appearing \emph{de novo} in the model. On the other hand choice of very large values would result in the lack of identification of variability in characteristics and thus make it difficult to infer natural selection (discussed later).

We next need an auxiliary relation $\Delta$ to determine that the differences due to the reproductive mutations are also bounded by $\delta_{rep\_mut}$. 

\begin{df}
$\Delta \subseteq E_{\T} \times E_{\T}$ s.t. $\forall e, e' \in E_{\T}$
\[
(e, e') \in \Delta \Leftrightarrow \left\{ \begin{array}{ll} 1.\, & \forall \mathit{char}_i \in \Upsilon\ . \mbox{ if } \mathit{char}_i \mbox{ has an ordering } \preceq_i \mbox{ then } D(e,e')[i] \preceq_i \delta_{rep\_mut}[i]\\
2.\, & \mbox{if } \exists s, s' \in \Sigma_{\T} \mbox{ s.t. } e \in E_s \mbox{ and } e' \in E_{s'} \mbox{ then } s' \geq s+1 \\
3.\, & (e, e') \in \R \Rightarrow (e, e') \not\in \Delta \\
\end{array} \right.
\]
\end{df}

Informally for $(e, e'$) to be in $\Delta$, their differences for each single characteristic $char_i$ must be bounded by $\delta_{rep\_mut}[i]$ and $e$ should not be recognized as mutating to $e'$. \\

Based on the thus established notion of ``causal'' relationships between entities and $\Delta$, we will define $\mathbf{AncestorOf}$ relation, which connects entities for which
an observer can establish descendance relationship across generations.

\begin{df}$\mathbf{AncestorOf}$ = $(C \ \cup \ \mathbf{R_{\delta_{mut}}})^+ \ \cap \ \Delta$ \end{df}

In this definition the transitive closure of $(C \ \cup \ \mathbf{R_{\delta_{mut}}})$ captures the observed causality ($C$) across multiple states even in cases when ``parent" entities might undergo mutational changes ($\mathbf{R_{\delta_{mut}}}$) before ``child" entities complete their ``epigenetic" maturation with possible reproductive mutations. Intersection with $\Delta$ ensures that causally related parent and child entities are not too different from each other, that is, reproductive mutational changes are under observable limit. 
For $e, e' \in E_{\T}$, $(e, e') \in$ $\mathbf{AncestorOf}$ when  $e$ is observed as an ancestor of $e'$.
\begin{figure}
\centering
\includegraphics[scale=0.7]{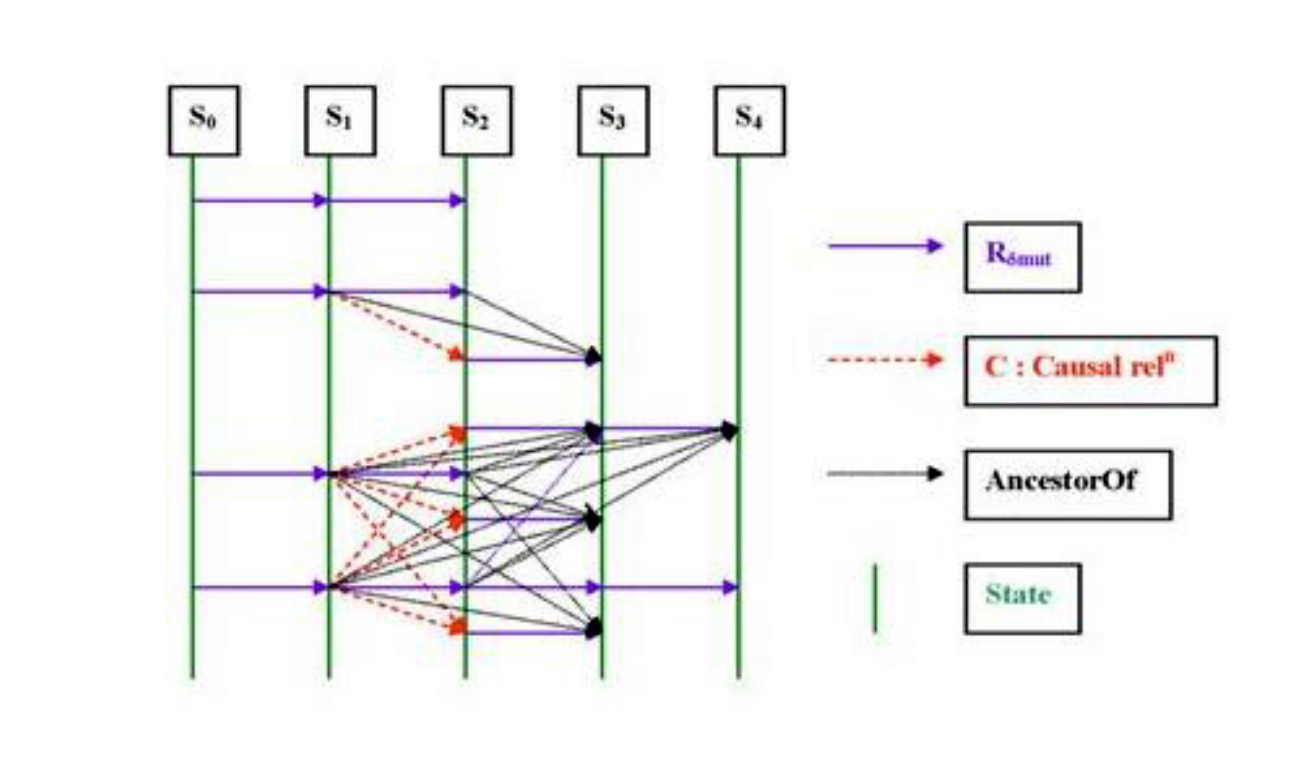}
\caption[Graphical view of evolutionary relations]{Graphical view of the relationships between entities in successive states. Recognition relation $\mathbf{Rec}$, Causal relation $C$, and $\mathbf{AncestorOf}$.} \label{fig:hassediagram}
\end{figure}

Figure~\ref{fig:hassediagram} depicts graphically the relationships between entities in successive states. Vertical lines represent the states ($S_0, S_1, S_2, S_3, S_4$). Various kinds of arrows represent different relationships: recognition relation $\mathbf{R_{\delta_{mut}}}$, causal relation $C$, and $\mathbf{AncestorOf}$. The end points of the arrows on state lines represent entities. 

To illustrate the point further let us also consider the case of {\sf Reflexive Autocatalysis}: In the simplest form, a reflexive  autocatalytic cycle is represented as a system of reaction equations: 
\begin{eqnarray*}
  A + X_1 &=& A_1 + Y_1 \\
  A_1 + X_2 &=& A_2 + Y_2 \\
  \vdots \\
  A_{n-1} + X_{n} &=& mA' + Y_n
\end{eqnarray*}
where $m$ copies of entity $A'$ are produced at the end and that entity $A'$ is a variation of entity $A$, i.e., $(A, A') \in \Delta$. Such autocatalytic cycles are supposed to be the chemical basis of biological growth and reproduction. Examples include the Calvin cycle, reductive citric acid cycle, and the formose system. Competing cycles of this sort can even undergo limited evolution, though they are supposed to have very limited heredity~\cite{ss97}.

In the current framework if an observer could determine the causal relations - $(A, A_1)$, $(A_1, A_2)$, $\ldots$, $(A_{n-1}, A')$ and if the entity $A$ does not undergo any changes before $A'$ is produced, that is, $(A, A) \in \mathbf{R_{\delta_{mut}}}$. Then $(C \ \cup\ \mathbf{R_{\delta_{mut}}})^+$ would contain $(A, A')$ so also would  $(\ (C \ \cup\ \mathbf{R_{\delta_{mut}}})^+ \ \cap \ \Delta)^+$ establishing the reproduction of $A$ through reflexive autocatalytic cycle and with variation.

\begin{theorem}~\label{th:ancof}
The definition of $\mathbf{AncestorOf}$ effectively formalizes the recognition of reproductive relationships under parental mutations together with reproductive mutations and epigenetic developments in the child entities. 
\end{theorem}
\begin{proof}
An observation process need to observe child entities so long that their epigenetic development unfolds completely. Since, in general, a priori limits cannot be assumed on the number of states required for such epigenetic development, this requirement of observations across states is captured using transitive closure - $(C \ \cup \ \mathbf{R_{\delta_{mut}}})^+$, where $\mathbf{R_{\delta_{mut}}}$ ensures that (mutational) changes in the parent entities and also the changes in the child entities during epigenetic development are accounted for. To see this,

Lets us assume that in a state $s$, a child entity $c$ was observed for the first time and (parent) entity $p$ present in the state $s-1$ was observed to be casually connected to it. Suppose that for the entity $c$ its epigenetic development unfolds through states $s+1, s+2, \ldots, s+r$ such that with changes owing to the development $c$ was observed as $c_{1}, c_{2}, \ldots, c_{r}$ in these states with $(c, c_1), (c_1, c_2), \ldots, (c_{r-1}, c_r) \in \mathbf{R_{\delta_{mut}}}$. Similarly suppose that parent entity $p$ has been undergoing mutations in these states and also in the states before observed as $p'_{1}, p'_{2}, \ldots, p'_k, p, p_{0}, \ldots, p_{r}$ such that $(p'_1, p'_2), \ldots, (p'_k, p), (p, p_{0}), \ldots,  (p_{r-1}, p_r) \in \R$. It is clear that $(C \ \cup \ \mathbf{R_{\delta_{mut}}})^+$ would contain $(p', c), (p'_1, c_1), \ldots, (p'_k, c_r)$, $(p, c), (p, c_1), \ldots, (p, c_r), \ldots, (p_r, c)$, $(p_r, c_1), \ldots, (p_r, c_r)$ among other tuples implying that the intersection of $(C \ \cup \ \mathbf{R_{\delta_{mut}}})^+$ with $\Delta$ would result in those tuples $(p_m, c_r)$, where $p_m$ and $c_r$ are sufficiently similar in their characteristic. Therefore if the resultant set $((C \ \cup \ \mathbf{R_{\delta_{mut}}})^+ \cap \Delta)$ is not empty, the observer can establish the reproductive relationship between entities $p$ and $c$ even under parental mutational changes and the epigenetic changes and reproductive mutations in the child entity.
\end{proof}

Using $\mathbf{AncestorOf}$ relation, we now can consider the cases of \emph{entity level reproduction} and \emph{Fecundity}:

\subsubsection*{Case 1: Entity Level Reproduction}

We consider the case where instances of individual entities can be observed as reproducing even though there might not be any observable increase in the size of the whole population. For a given  simulation $\T$ of the model, an observer defines the following  $\mathbf{Parent_{\Delta}}$ relation:
\begin{df}
\begin{equation*}
\begin{split}
\mathbf{Parent_{\Delta}} = \{ & (p, c) \in \mathbf{AncestorOf} \mid\\
& \not\!\exists e \in E_{\T}\ .\  [(p, e) \in \mathbf{AncestorOf}
\wedge (e, c) \in \mathbf{AncestorOf}]\}
\end{split}
\end{equation*}
\end{df}
The condition in defining $\mathbf{Parent_{\Delta}}$ is used to ensure that $p$ is the immediate parent of $c$ and thus there is no intermediate ancestor $e$ between $p$ and $c$.  Notice that $(p, c) \in \mathbf{Parent_{\Delta}} \Leftrightarrow (p, c) \in \left(\left({\R}^+ \cup C \cup {\R}^+\right) \cap \Delta\right)$.  Using $\mathbf{Parent_{\Delta}}$ relation, in order for the observer to establish reproduction in the model, the following axiom should be satisfied:
\begin{ax}[\textbf{Reproduction}]
$\exists \textit{state sequence}\ \T \ .\
\mathbf{Parent_{\Delta}} \neq \emptyset$
\end{ax}
This means, if there is reproduction in the model, then there should exists a simulation $\T$ of the model, where at least one instance of reproduction is observed. 

Notice that, for every $(p, c) \in \mathbf{Parent_{\Delta}}$, some other $(p', c') \in \mathbf{AncestorOf}$ where $p$ has been observed to change to $p'$ and $c$ to $c'$ through a sequence of states would also be included in the $\mathbf{Parent_{\Delta}}$. Therefore, in the following, we limit the attention to temporally least such parent-child pairs in $\mathbf{Parent_{\Delta}}$: 
$$\mathbf{Parent_{\Delta}^{min}} = \{ (p, c) \in \mathbf{Parent_{\Delta}} \mid\ \not\!\exists (p', c') \in \mathbf{Parent_{\Delta}} .\  [(p', p) \in \mathbf{R^+_{\delta_{mut}}} \vee (c', c) \in \mathbf{R^+_{\delta_{mut}}}]\}$$
  
\subsubsection*{Case 2: Population Level Reproduction - Fecundity}

Though entity level reproduction is essential to be observed, for natural selection it is the population level collective reproductive behavior (fecundity), which is significant owing to the \emph{carrying capacity} of the environment. Since carrying capacity is an limiting constraint on the maximum possible size of population, an observer needs to establish that there is no perpetual decline in the size of the population. In other terms for all generations, there exists a future generation that is of the same size or larger. This allows cyclic population sizes where the cycle mean grows (or stays steady) over time. Also in case of fecundity, an observer need not to observe all the parents in the same state, nor do children need to be observed in the same states of the model. Formally we require the observer to establish Fecundity by satisfying the
following axiom: 
\begin{ax}[\textbf{Fecundity}]
There exist statistically significant number of different generations of reproducing entities in temporal ordering $G_1, G_2, \ldots, G_L$ such that $(\forall G_{i < L} \subseteq E_{\T})(\exists G_{j>i} \subseteq E_{\T})\ .\ |G_j| \geq |G_i|$ where $G_j = \{c \in E_{\T} \mid \exists a \in G_i\ .\ (a, c) \in \mathbf{AncestorOf}\}$
\end{ax}
Informally, the axiom states that for every generation of reproducing entities ($G_i$), in future there exist generation of its descendant entities ($G_j$) such that the size of descendant generation must be equal or more than current generation. Note that the granularity of the time for determining generations is entirely dependent on the design of the model and the observation process. 

Let us next prove that, above formulations successfully yield the following claim in case of our example of Langton's CA model:

\begin{figure}[hbtp]
\begin{center}
\includegraphics[scale=0.8,trim=0 0 0 0,clip=]{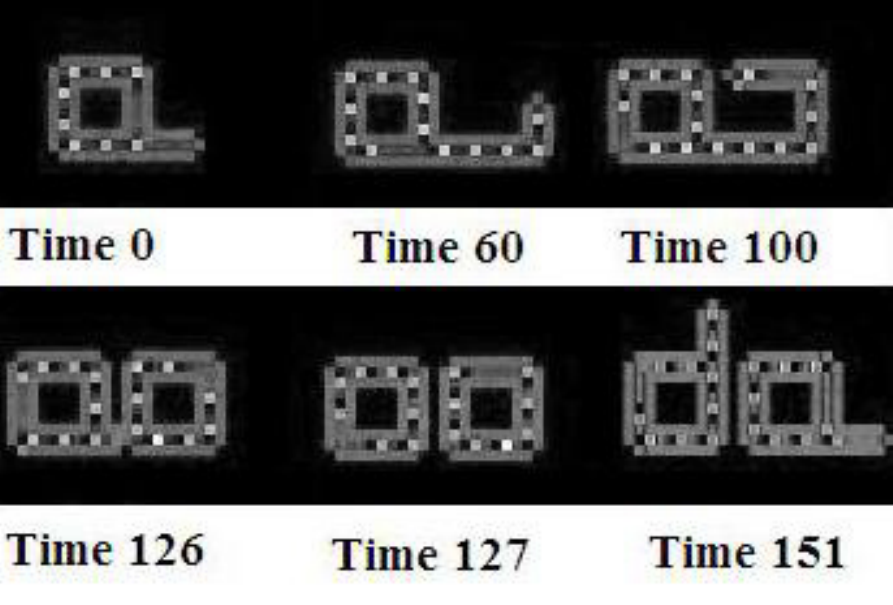}
\caption[Self-replication in Langton loops]{Self-Reproduction in Langton loops; screen shots from~\cite{Sayama2005}} \label{ftr1}
\end{center}
\end{figure}

\begin{lemma}\label{lemma:ca-entity-rep}
\textbf{Axiom of Reproduction} and the \textbf{Axiom of Fecundity} are satisfied by the entities and abstractions on Langton Loops defined above. 
\end{lemma}
\begin{proof}
These two axioms can be established by the observer in a specific state sequence as exemplified in Figure~\ref{ftr1} and Figure~\ref{ftr2} by repeatedly applying the recognition relation $\mathbf{R_{\delta_{mut}}}$ when entities are changing in number and states of cells (retaining the pivots) and applying the causal relation when a parent entity splits (e.g. at Time=$127$). The relation $\Delta$ connects the initial parent entity and the child entity at Time=$151$.

With respect to Figure~\ref{ftr1},  an entity is identified at Time=$0$ with associated pivot. Between time steps $[1\ldots126]$  entity changes in number and states of its cells but the pivot remains the same, hence as per the definition of $\mathbf{R_{\delta_{mut}}}$, the observer can recognize the entity in these successive states.
At Time=$127$, the (parent) entity is observed to be splitting into two identical copies. One of these is again recognized as the original parent entity because of its pivot and the second entity would be claimed to be causally related with the parent entity as per the definition of $C$. To see this, notice that the parent entity at Time=$126$ contains all the cells of the child entity appearing at Time=$127$, which satisfies the definition of $C$. Between time steps $128$ and $151$ both parent and child entities undergo changes in the number and states of their cells but their  pivots remain fixed. Hence they can again be recognized. Finally at Time=$151$ the child entity becomes identical to the original parent entity, therefore the parent entity at Time=$0$ and the child entity at Time=$151$ are related using $\Delta$. The transitive closure finally give us the final descendance relationship between the parent and the child entity.
\end{proof}

\begin{figure}[hbtp]
\begin{center}
\includegraphics[scale=0.7,trim=0 0 0 0,clip=]{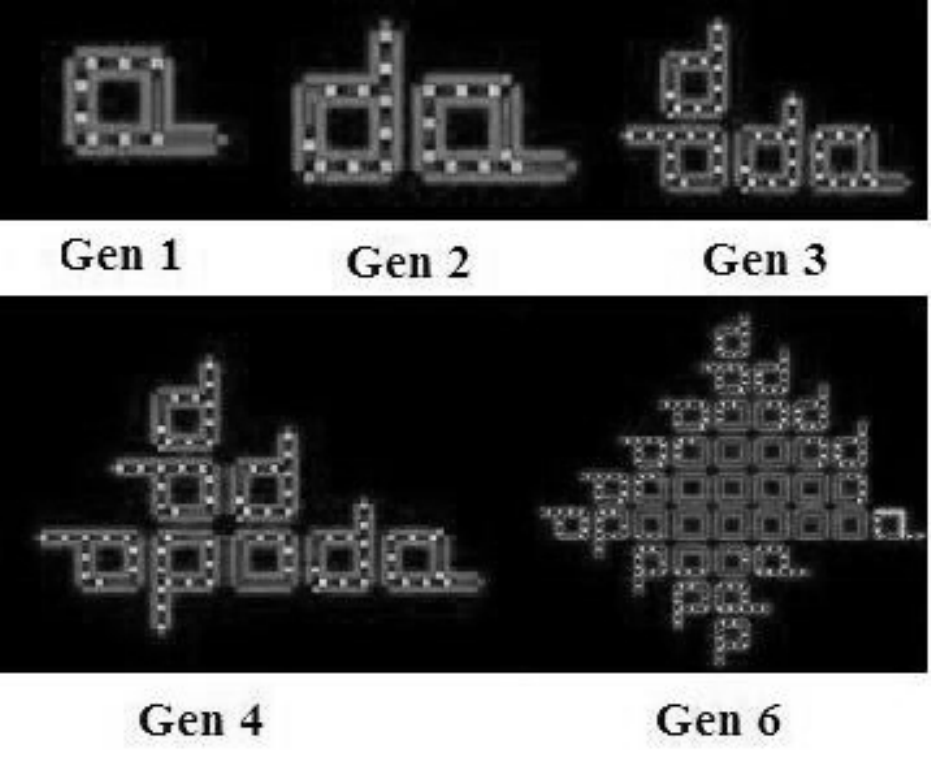}
\caption[Fecundity in Langton loop population]{Fecundity across generation in a population of Self Replicating Langton Loops; screen shots from~\cite{Sayama2005}} \label{ftr2}
\end{center}
\end{figure}

\subsubsection{Heredity}
\label{heredityformalized}


Heredity, yet another precondition for evolution, can in general be observed in two different levels: Syntactic level and Semantic level. On \emph{syntactic level}, entity level inheritance is implied by the structural proximity between parents and their progenies ranging over several generations - though in case of continuous structural changes in the parental entities and epigenetic development in progenies, this would require an observer to establish structural similarities over a range of states as discussed earlier with the definition of $\mathbf{AncestorOf}$ relation. Also for syntactic inheritance to persist, design of the model needs to ensure that environment, which controls the reaction semantics of entities, remains approximately constant over a course of time so that structural similarities also result into continued reproductive behavior. 

Difficulty arises primarily on the level of multi parental reproduction - in this situation an observer might have to stipulate some kind of gender types and might have to relax the mechanism of recognizing the parent-child relationship in a way as happens for example in case of organic life, where male-female reproductive process (often) gives birth to a progeny belonging to ``only" one gender type. In such a case, for heredity, an observer need to ensure that, {\sf over a course of time all the gender types are sufficiently produced in the population.}

On the other hand it is also possible to observe inheritance on the semantic level (ignoring structural differences) in terms of \emph{semantic relatedness} between entities, whereby an observer can observe that progenies and their parental entities exhibit similarities in their  behaviors (e.g., reproductive) under near identical set of environments. This in turn would require an observer to identify the possible sequences of observable reactions between existing entities, which appear to be yielding new set of entities (children) and in the child generation as well there exist a similar observable reproductive process, which enables the (re)production of entities. Such an observation would enable the observer to abstract the reproductive processes currently operational in the model. The inherent difficulties in this view are obvious - in essence an observer needs to abstract the reproductive semantics from the observable reactions in the model, which in turn might require non trivial inferences in absence of the knowledge of the actual design of the model.

Considering the case of real-life from an observation view point, semantic view is in fact an abstraction over all the
reproductive processes existing across various species and levels including the case of bacterial organisms, where next generation of bacteria may contain a mix of genetic material from various parental bacterium of previous generation through the process of horizontal transmission. So while in case of syntactic inheritance an observer would only be able to establish inheritance across organisms belonging to same species, using semantic view, he could expand his horizon to the all organic life as a whole. 


However, heredity as a mechanism of preservation of syntactic structures, appears to be crucial for those AES models where entities have very limited set of reproductive variations possible, that is, where environment supports only rare forms of entities to reproduce and any changes in the syntactic structure of these reproductive entities may result in the elimination of the reproductive capability. Real-life on earth as well as the model of the Langton loops are definitive examples where most of the variations in the genetic structure, or the loops geometry/transition rules result in the loss of reproductive/replicative capabilities. 

Also heredity usually requires further mechanisms to reduce possible undoing of current mutations in future generations owing to new mutations. Therefore, in order to establish inheritance in AES models, sufficiently many generations of reproducing entities need to be observed to determine that the number of parent-child pairs where certain characteristics (both syntactic and semantic) were inherited by child entities without further mutations is significantly larger than those cases where mutations altered the characteristics in the child entities. We can express it as the following axiom:

\begin{ax}[\textbf{Heredity}]~\label{axiom:heredity}
Let a statistically large observed subsequence of a run $\T$:
$$\Omega = \langle s_m, \ldots s_r
\rangle , m \ll r$$ Consider $$\mathbf{Parent_{\Delta}^{\Omega}} =
\{(e, e') \in \mathbf{Parent_{\Delta}^{min}} \mid \exists s_l, s_r \in \Omega\ \textit{s.t.}\ 
e \in E_{s_l} \wedge e' \in E_{s_r}\}$$ to be the set of all parent-child pairs
observed in $\Omega$. Again let $$\mathbf{Inherited_{\Omega}^i} = \{(e, e') \in \mathbf{Parent_{\Delta}^{\Omega}} \mid \exists
\mathit{char}_i \in \Upsilon\ .\  D(e, e')[i] = 
0_{\mathit{{diff}_i}} \}$$ to be the set of those cases of reproduction where $i^{th}$ characteristics were inherited without (further) mutations.
High degree of inheritance for $i^{th}$ characteristics $\mathit{char}_i$ implies that
$$\displaystyle \lim_{r \rightarrow \infty}|\mathbf{Parent_{\Delta}^{\Omega}}|/|\mathbf{Inherited_{\Omega}^i}| \leadsto 1$$ 
For syntactic inheritance to be observed in a population of entities, we should have some such characteristics
which satisfy this condition.
\end{ax}

The axiom of heredity together with the axiom of preservation of reproduction under mutation is to ensure that reproductive variation is maintained and propagated across generations. 

\subsubsection{Natural Selection}
\label{natselformalized}

There are several existing notions of selection in the literature on evolutionary theory \cite{Fut98,Ridley96,Ridley97,SS00,Mah97,Kimura83}. In the current framework natural selection is considered as a \emph{statistical inference} of \emph{average reproductive success}, which should be established by an observer on the population of self reproducing entities over an evolutionary time scale i.e., over statistically large number of states in a state sequence. Other notions of selection using fitness, adaptedness, or traits etc. are rather intricate in nature because these concepts are relative to the specific abstraction of ``common environment'' shared by entities and ``the environment-entity interactions'', which are the most basic processes of selection. Nonetheless selecting appropriate generic abstraction for these from the point of view of an observation process is not so simple. Therefore following the idea from~\cite[page 19]{selection-book1} that on evolutionary scale rate of reproduction is the the only attribute selected directly and characteristics affecting the rate of reproduction are selected only indirectly, a more straightforward approach is considered here whereby reproductive success is assumed to be an indicator of better adaptedness or fitness. We thus define the following (necessary) axioms for the natural selection as generally discussed in the literature~\cite{SS00}:

\begin{ax}[\textbf{Observation on Evolutionary Time Scale}]~\label{axiom-ns1} 
An Observer must observe statistically significant population of different reproducing entities for statistically large number of states in a run $\T$. 

Formally, for a statistically large subsequence $\Omega$ of $\T$, $\Omega = \langle s_m, \ldots s_r \rangle , m \ll r$, the observer determines set of reproducing entities $\mathbf{\Lambda_{min}}$ as follows:
\begin{eqnarray*}
\mathbf{\Lambda_{min}} &=& \Lambda \setminus \left\{p \in \Lambda \mid \exists p' \in \Lambda . (p', p)\in \mathbf{R^+_{\delta_{mut}}}\right\} \ \mbox{where} \\
\displaystyle \mathbf{\Lambda} &=& \bigcup_{s_j \in \Omega}SR(s_j) \mbox{ and }\\
SR(s_j) &=& \left\{p \in E_{s_j} \mid \exists c \in E_{\Omega}\ .\  (p, c) \in \p \right\} \\
&& \  \ \ \ \mbox{ is the set of all reproducing entities in state $s_j \in \Omega$.}
\end{eqnarray*} 
\end{ax}

\begin{ax}[\textbf{Sorting}]~\label{axiom-ns2} Entities in $\mathbf{\Lambda_{min}}$ should be different with respect to characteristics in $\Upsilon$ and there should exist differential rate of reproduction among these reproducing entities. Rate of reproduction for an entity is the number of child entities it reproduces before undergoing any mutations beyond observable limit. In other words, $ror: \mathbf{\Lambda_{min}} \rightarrow \N$ is defined as following: 
\begin{eqnarray*}
ror(p) &=& |Child_p| \ \mbox{where,} \\
Child_p &=& \left\{ c \in E_{\Omega} \mid\ \exists p' \in E_{\Omega}\ .\  \left[(p, p') \in \mathbf{R^*_{\delta_{mut}}}\right]\ \wedge \left[(p', c) \in \p \right]\right\} 
\end{eqnarray*}
Variance in the characteristics of entities in $\mathbf{\Lambda_{min}}$ is estimated as: $\forall char_i \in \Upsilon$
\begin{eqnarray*}
{\sigma}^2_{i} &=& \left[\sum_{e \in \mathbf{\Lambda_{min}}}(e[i])^2\right] - \left[|\mathbf{\Lambda_{min}}|.\mu^2_{i}\right] \mbox{ where,} \\
\mu_{i} &=& \frac{\sum_{e \in \mathbf{\Lambda_{min}}}e[i]}{|\mathbf{\Lambda_{min}}|}
\end{eqnarray*}
Similarly variance in the rates of reproduction is estimated as:
\begin{eqnarray*}
{\sigma}^2_{ror} &=& \left[\sum_{e \in \mathbf{\Lambda_{min}}}(ror(e))^2\right] - \left[|\mathbf{\Lambda_{min}}|.\mu^2_{ror}\right] \mbox{ where,} \\
\mu_{ror} &=& \frac{\sum_{e \in \mathbf{\Lambda_{min}}}ror(e)}{|\mathbf{\Lambda_{min}}|} 
\end{eqnarray*}
\end{ax}
The above two axioms though necessary are not sufficient to establish natural selection since these cannot be used as such to distinguish between natural selection with neutral selection~\cite{SS00}. The following axioms are therefore needed to sufficiently establish natural selection.

\begin{ax}[\textbf{Heritable Variation}]~\label{axiom-ns3} There must exist variation in the inherited mutations in the population of $\mathbf{\Lambda_{min}}$. Formally, let 
$$\mathbf{Ch_{mut}} = \left \{ c \in \mathbf{\Lambda_{min}} \mid\  \exists p \in \mathbf{\Lambda_{min}}\ .\  \left[(p, c) \in \p\right] \wedge \left[\exists \mathit{char}_i \in \Upsilon\ .\  0_{\mathit{diff}_i} \prec D(p, c)[i]\right] \right \}$$ 
be the set of child entities carrying reproductive mutations. Let $\mathbf{Var\_Ch_{mut}}$ be the largest subset of $\mathbf{Ch_{mut}}$ consisting of those child entities which are unique with respect to at least one characteristic in $\Upsilon$, that is, $$\forall c, c' \in \mathbf{Var\_Ch_{mut}} \ .\  \exists \mathit{char}_i \in \Upsilon \ .\  0_{\mathit{diff}_i} \prec D(c, c')[i]$$ Then axiom of heritable mutation demands that $$\displaystyle \lim_{r \rightarrow \infty}\frac{|\mathbf{Var\_Ch_{mut}}|}{|\mathbf{\Lambda_{min}}|} \leadsto 1$$ Informally, which means, a significant fraction of the population of all reproducing entities is unique in its characteristics.
\end{ax}

\begin{ax}[\textbf{Correlation}]~\label{axiom-ns4} There must be non zero correlation between heritable variation and differential rate of reproduction.
Formally, a linear correlation coefficient $r_{i}$ (also known as the Pearson's product moment correlation coefficient~\cite{statbook_cox}) between the characteristic $char_i \in \Upsilon$ and the rates of reproductions for the entities in $\mathbf{Var\_Ch_{mut}}$ is defined as
\begin{eqnarray}
r_{i} &=& \frac{\displaystyle \sum_{c \in \mathbf{Var\_Ch_{mut}}}\left(c[i] - \mu_i\right)\left(ror(c) - \mu_{ror}\right)}{\sqrt{\displaystyle \sum_{c \in \mathbf{Var\_Ch_{mut}}}(c[i] - \mu_i)^2 \displaystyle \sum_{c \in \mathbf{Var\_Ch_{mut}}}(ror(c) - \mu_{ror})^2}} \mbox{ where,} \\
\mu_{i} &=& \frac{\displaystyle \sum_{c \in \mathbf{Var\_Ch_{mut}}}c[i]}{|\mathbf{Var\_Ch_{mut}}|} \nonumber \mbox{ and} \\
\mu_{ror} &=& \frac{\displaystyle \sum_{c \in \mathbf{Var\_Ch_{mut}}}ror(c)}{|\mathbf{Var\_Ch_{mut}}|}  \nonumber
\end{eqnarray}
Therefore a non zero correlation between heritable variation and differential rate of reproduction would require that $$\exists char_i \in \Upsilon. r_i > 0$$
\end{ax}
Informally, this means as the values of characteristics inherited by the child entities change, rate of reproduction also changes. Based upon the environmental pressures with respect to a particular characteristics, rate of reproduction might either increase or decrease as the characteristic changes.

The last two axioms state that there must be significant variation in population of entities with respect to their characteristics, which must be maintained for evolutionarily significant periods and that this variation must be caused by the differences in inheriting mutations from the parent entities, which in turn directly  affect the rate of reproduction.

In case of real-life, populations show variation in the characters on all levels, from gross morphology to genetic sequences. These variations in characters are generated
randomly by mutation and recombination~\cite{Ridley97} on genetics material. Based upon the variation in these characters, entities in natural population also vary in their reproductive successes.

Yet another important constraint from the evolutionary perspective is that reproduction in a model should not entirely cease because of the (harmful) mutations. Though this constraint is implicitly captured in the axioms~\ref{axiom-ns1} and \ref{axiom-ns2}, we can still restate it below primarily since this weaker version enables us to directly argue for the reasons of the absence of evolutionary behavior in a model:

\begin{ax}[\textbf{Preservation of Reproduction under Mutations}]~\label{axiom:continuation} 
Some mutations do preserve reproduction. Formally, 
\begin{eqnarray*}
(\forall s \in E_{\Omega})\ &.&\ SR(s) \neq \emptyset  \Rightarrow  (\exists e \in T)\ .\  ror(e) > 0 \mbox{ where,} \\
T &=& \displaystyle \bigcup_{p \in SR(s)} [Child_p] \cup [\{p' \mid (p, p') \in \mathbf{R^+_{\delta_{mut}}}\}]
\end{eqnarray*}
\end{ax}
Informally, this means, if there exist reproductive entities in a state $s$, either some mutants of these entities or their children should continue reproducing further.

\section{Computational Complexity of Observing Evolution}\label{e-complexity}

In the next few sections, we will estimate upper bounds on the worst case time complexity to the problem of establishing various axioms dealing with various evolutionary components in the framework for arbitrary AES models. We will also estimate corresponding bounds for the case of Langton's CA based model to demonstrate that often generic upper bounds established for arbitrary AES models may be refined for the case of specific models at hand. 

The estimates for space complexity will be addressed only sometimes. Primary reason for that is that space (memory) requirement is often dependent upon the actual model at hand, the syntactic nature of the entities as determined by an observation process, and often is linear w.r.t. the total number of states observed. 

An important problem to be considered while providing estimates on the computational complexities is that observed state progression during simulations might not correspond to the actual underlying reaction semantics for a specific entity. In other words observed states during simulations progress according to the underlying updating rules for the model, which determine which subset of entities would react in any state. Since an automated discovery of the updation rules in not what is considered here, and automated discover of the evolutionary behavior is considered to be independent of the underlying updation semantics, we assume that all those entities, which are enabled to react in each state, are indeed allowed to react. In cases where it is not true, an observation process stores state subsequences of certain finite lengths where all (or most of) the enabled entities have been observed to react and then merges all the states in these subsequences into a single meta state, which reflect the effect that most of those entities which can react have actually reacted. However, which of the entities will react is still determined by the underlying updation rules and to be followed by observation process as well.  

\subsection{Computational Complexity of Entity Recognition}\label{complexity}

Following results appear in~\cite{Misra09a}:

\begin{theorem}\label{complexity-entity}
The problem of entity recognition using structural (syntactic) constraints is NP-hard.
\end{theorem}   

In the worst case, in a state $s$ of size $n$, it would take no less than $\cO(2^n)$ steps to construct the entity set consisting of all possible multisubsets of $s$ in it. In practice also, during initial stages of the simulations, there might not be any hard constraint defining the entity structure and all possible multisubsets may represent valid entities. 
If the size of a run $\T$ is $r$, then entity recognition using structural constraints in all the states ${s_0}, {s_1}, \ldots, {s_r}$ may require in the worst case $\cO(r2^n)$ steps.

Let us also consider the case, where entities do not have overlapping structures. In that case, the entity set $E_s$ would  be a partition of a subset of $s$. Following results sets upper bound on the worst case computational complexity of entity recognition in this case:
\begin{theorem}\label{complexity-entity-nooverlap}
Generating entity sets consisting of entities with non overlapping structures, for all the states of a run $\T$ of size $r$, could take time steps upper bounded by $\cO(rn2^n)$, where expected size of a state is $\cO(n)$.
\end{theorem}   

For the specific cases of recognizing Langton loops in a CA model and Lambda entities in the AlChemy, we have the following bounds: 


\begin{lemma}\label{complexity-en-loop}
The worst case computational complexity of observing Langton loops in a CA model simulation is upper bounded by $\cO(n)$, where $n$ is the size of the lattice.
\end{lemma}  

And 

\begin{lemma}\label{complexity-en-alchemy}
The worst case computational complexity of identifying and tagging all the lambda entities in a given state in AlChemy is upper bounded by $\cO(n \log n)$, where $n$ is the number of $\lambda$ terms in a state.
\end{lemma}

\subsection{Computational Complexity of Observing Evolutionary Components}

We can now discuss some of the computational complexity theoretic aspects of observing various components of evolution. We will assume that all the states in a simulation are of comparable size and use $\cO(n)$ as the size of any state. Also we will use the following notations: \\\\
$t_c$:   expected number of time steps required to determine membership of an entity pair in the relation $C$.\\
$t_\Delta$:  expected number of time steps required to determine membership of an entity pair in the relation $\Delta$.\\
$t_{\delta_{mut}}$:  expected number of time steps required to determine membership of an entity pair in the relation $\R$.\\
$t_=$:  expected number of time steps required to compare two entities for equality checking.\\
$t_D$:  expected time steps required to compute function $D$ to check the equality (or inequality) of the characteristics of two entities.

\subsubsection{Computational Complexity of Observing Entity Level Reproduction} 

Establishing the case for the entity level reproduction in the simplest case where there are no epigenetic developments in the child entities minimally demands identifying a single instance of a reproducing entity and its progeny in the next state during one simulation. In other words, suppose an observer needs to determine that an entity $p$ in a state $s$ is an instance of a reproducing entity. Then the observer needs to establish that under its specified definition of the causal relation $C$, there exists another entity $c$ in the state $s+1$ such that $(p, c) \in C$ and that the reproductive mutations in $c$ with respect to $p$ are bound by $\delta_{rep\_mut}$, i.e., $(p, c) \in \Delta$, and that there does not exist any other entity in the state $s$, which is recognized as mutating to $c$. 

In this process, let determining membership the membership of $(p, c)$ in $C$, $\Delta$, and $\R$ may take $t_c$, $t_{\Delta}$, and $t_{\delta_{mut}}$ steps at worst respectively, then the whole process would at worst take $N_p^{(s)} = t_c + t_{\Delta} + |E_s|t_{\delta_{mut}}$ steps. Since for a state $s$, such a reproducing instance may not be found quickly, in the worst case all the entities in the state $s$ might need to be assessed under these steps. Therefore search for an reproducing instance in a state $s$ may take at worst 
\begin{eqnarray*}
\displaystyle \sum_{p \in E_s}N_p^{(s)} &=& |E_s|N_p^{(s)} \\ 
&=& |E_s|(t_c + t_{\Delta} + |E_s|t_{\delta_{mut}}) \\ 
&\leq& 2^n(t_c + t_{\Delta} + 2^nt_{\delta_{mut}})\\ 
&=& \cO(2^n \max\{t_c, t_{\Delta}, t_{\delta_{mut}}2^{n}\}) 
\end{eqnarray*}
 steps, where $|E_s| \leq 2^n$. Since finding such a state $s$, where a reproducing entity may be present itself may require search into a potentially large state subsequence of a run, it might take $$\cO(\eta)*\cO(2^n\max\{t_c, t_{\Delta}, t_{\delta_{mut}}2^{n}\}) = \cO(\eta 2^n \max\{t_c, t_{\Delta}, t_{\delta_{mut}}2^{n}\})$$ steps to establish the entity level reproduction, where $\eta$ is the number of states in the state subsequence used in the search assuming that all the states are of comparable sizes. Therefore we have

\begin{theorem}\label{complexity-entity-rep}
Given the sets of entities in each state, additional time steps required for observing spontaneous entity level reproduction, i.e., reproduction without epigenetic development in the child entities and mutational changes in the parent entities, in an AES is upper bounded by $\cO(\eta 2^n \max\{t_c, t_{\Delta}, t_{\delta_{mut}}2^{n}\})$, where $\eta$ is the number of states observed before first instance of entity level reproduction is recognized.
\end{theorem}   

The case where entities do not have overlapping structures, total number of entities in a state are restricted by the number of atomic structures, that is, $|E_s| \leq n$. Therefore we have the following corresponding corollary:
  
\begin{cor}\label{complexity-entity-rep-nooverlap}
Given the sets of entities in each state, additional time steps required for observing spontaneous entity level reproduction in an AES where entities do not have overlapping structures is upper bounded by $$\cO(\eta n\max\{t_c, t_{\Delta}, t_{\delta_{mut}}n\})$$
\end{cor}

Next let us consider the general case of entity level reproduction with epigenetic developments in child entities and mutational changes in the parent entities. Towards that we will prove the following:

\begin{theorem}\label{complexity-en-epi_rep}
Given the sets of entities in each state, additional time steps required for establishing an entity level replication is upper bounded by $\cO\left(r2^n\max\left\{t_{\delta_{mut}}, t_{c}2^{n}, t_{\Delta}2^{n}, t_{=}r^32^{3n}\right\}\right)$, where $r$ is the number of states before observing the occurrence of first instance of reproduction.
\end{theorem} 
\begin{proof}
Consider the case where first such entity level reproduction is recognized. As in the proof for the theorem~\ref{th:ancof}, lets us assume that in a state $s \geq 1$, such a child entity $c$ was observed for the first time and (parent) entity $p$ present in the state $s-1$ was observed to be casually connected to it. Suppose that for entity $c$ its epigenetic development unfolds through states $s+1, s+2, \ldots, s+k$ such that with changes owing to the development $c$ was observed as $c_{1}, c_{2}, \ldots, c_{k}$ in these states with $(c, c_1), (c_1, c_2), \ldots, (c_{k-1}, c_k) \in \mathbf{R_{\delta_{mut}}}$. Similarly suppose that parent entity $p$ has been undergoing mutations in these states and also in the states before observed as $p'_{1}, p'_{2}, \ldots, p'_j, p, p_{0}, \ldots, p_{k}$ such that $j + 2 < s$ and $(p'_1, p'_2), \ldots, (p'_j, p), (p, p_{0}), \ldots,  (p_{k-1}, p_k) \in \R$. It is clear that $(C \ \cup \ \mathbf{R_{\delta_{mut}}})^+$ would contain $(p', c), (p'_1, c_1), \ldots, (p'_j, c_k)$, $(p, c), (p, c_1), \ldots, (p, c_k), \ldots, (p_k, c)$, $(p_k, c_1), \ldots, (p_k, c_k)$ among other tuples implying that the intersection of $(C \ \cup \ \mathbf{R_{\delta_{mut}}})^+$ with $\Delta$ in the state $s+k$ would result in those tuples $(p_m, c_k)$, where $p_m$ and $c_k$ are sufficiently similar in their characteristic. 

As discussed before, for any state $s$ in the state subsequence, $|E_{s}| \leq 2^n$. As the sets $\R$ and $C$ are constructed recursively in each observed state $s$ considering the entity sets $E_{s}$ and $E_{s-1}$, $\R$ being an injective function, at most $N_{\delta_{mut}}^{(s)} = \min\{|E_{s-1}|, |E_{s}|\} \leq 2^n$ number of entity pairs $(e, e') \in E_{s-1} \times E_{s}$ could be added to it. Therefore recursive construction of $\R$ till state $s$ would cost at most $s*t_{\delta_{mut}} * N_{\delta_{mut}}^{(s)} = \cO(st_{\delta_{mut}}2^n)$ steps, where $t_{\delta_{mut}}$ is the time complexity to satisfy the constraints for $\R$. However $C$ being a many-to-many relation could potentially have at most $N_{C}^{(s)} = |E_{s-1}|*|E_{s}| \leq 2^{2n}$ new entity pairs added to it and recursive construction of $C$ would therefore cost at most $s*t_{c} * N_{C}^{(s)} = \cO(st_{c}2^{2n})$ steps till state $s$.   

Also for building the relation $\Delta$ recursively after each observed state, it would require to determine all those entity pairs $(e, e')$ such that $e \in E^{(s-1)}$ and $e' \in E_{s}$ and $D(e, e') \prec \delta_{rep\_mut}$, where $E^{(s-1)} = \bigcup_{j = 0}^{j = s-1} E_j$. This would computationally require at most $\cO(t_{\Delta}*|E_{s}|*|E^{(s-1)}|) = \cO(t_{\Delta}s2^{2n})$ steps. 

The computationally most expensive process is recursively building the relation $\ancof$. Using the breadth-first search based diagraph traversal algorithm~\cite[pp. 558–-565]{algorithms}, the transitive closure $(C \cup \R)^+$ in  state $s$ can be computed in $\cO\left(|E^{(s)}|*\left(|E^{(s)}| + |(C \cup \R)|\right)\right) = \cO(s^22^{3n})$ steps because in the state $s$, relation $(C \cup \R)$ would contain at most $s2^{2n}$ pairs and total size of the entity sets $E_0, E_1, \ldots, E_{s}$ is $|E^{(s)}| \leq (s+1)2^{n} = \cO(s2^n)$. Next, using the brute force method for set intersection, computing relation $\left((C \cup \R)^+ \cap \Delta \right)$ would take at most 
\begin{eqnarray*}
\cO(t_{=}*|(C \cup \R)^+|*|\Delta|) &=& \cO(t_{=}*s^22^{2n}*s^22^{2n})\\ &=& \cO(t_{=}s^42^{4n})  
\end{eqnarray*} 
steps, where $|\Delta| \leq |C^+| \leq {|E^{(s)}| \choose 2} \leq s^22^{2n}$ and if $t_{=}$ is the number of steps required for computing equality check between any two entities in $E_{\T}$, equality check between entity pairs $(e_1, e_2) \in (C \cup \R)^+$ and $(e'_1, e'_2) \in \Delta$ would take at most $2t_=$ steps. 

Therefore constructing the set $\ancof$ till the state $s$ would take time steps upper bounded by  
\begin{eqnarray*}
\Psi^{(s)}_{\ancof} &=& \cO(t_{\delta_{mut}}r2^n) + \cO(t_{c}r2^{2n}) + \cO(t_{\Delta}r2^{2n}) + \cO(r^22^{3n}) + \cO(t_{=}r^42^{4n})\\ &=& \cO\left(r2^n\max\left\{t_{\delta_{mut}}, t_{c}2^{n}, t_{\Delta}2^{n}, t_{=}r^32^{3n}\right\}\right)
\end{eqnarray*} 

Finally, the very first state, where the relation $\ancof$ would be non empty would be the state where relation $\mathbf{Parent_\Delta}$ by its very definition would also become non empty. This non emptiness check can be performed in constant time $\cO(1)$. 

Therefore, we can conclude that the computational complexity of the overall process of establishing the entity level reproduction would be upper bounded by 
\begin{equation*}
\Psi^{(s)}_{\ancof}  + \cO(1) = \Psi^{(s)}_{\ancof}
\end{equation*} 
time steps, where $r = s+k$. 
\end{proof}
The following corollaries immediately follow as a consequence of the process of constructing the relation $\ancof$:

\begin{cor}\label{complexity:parent-delta}
Given the sets of recognized entities in each state, the worst case computational complexity of constructing relation $\mathbf{Parent_\Delta}$ till $r^{th}$ state is upper bounded by $$\cO\left(r2^n\max\left\{t_{\delta_{mut}}, t_{c}2^{n}, t_{\Delta}2^{n}, t_{=}r^32^{3n+1}\right\}\right)$$
\end{cor}
\begin{proof}
Having constructed the relation $\ancof$ recursively, a brute force construction of the relation $\mathbf{Parent_\Delta}$ may require additional $t_=|\ancof|^2$ steps. This is because for each pair $(p, c) \in \ancof$, checking whether there exist pairs $(p, e)$ and $(e, c)$ also in $\ancof$ may take in the worst case $|\ancof|*t_=$ steps, and so doing it for each such pair in $\ancof$ would take at most $|\ancof|*(|\ancof|*t_=) = t_=|\ancof|^2$ steps. We also have, $|\ancof| \leq |\Delta| \leq |C^+| \leq r^22^{2n}$. Therefore constructing $\mathbf{Parent_\Delta}$ may take time steps upper bounded by 
\begin{eqnarray*}
\Psi^{(s)}_{\ancof} +  t_=|\ancof|^2 &\leq & \Psi^{(s)}_{\ancof} + t_=(r^22^{2n})^2 \\ 
&=& \cO\left(r2^n\max\left\{t_{\delta_{mut}}, t_{c}2^{n}, t_{\Delta}2^{n}, t_{=}r^32^{3n}\right\}\right) + \cO(t_=r^42^{4n})\\
&=& \cO\left(r2^n\max\left\{t_{\delta_{mut}}, t_{c}2^{n}, t_{\Delta}2^{n}, t_{=}r^32^{3n+1}\right\}\right)
\end{eqnarray*}
\end{proof}

\begin{cor}\label{complexity:parent-delta-min}
Given the sets of entities in each state, the worst case computational complexity of constructing relation $\mathbf{Parent^{min}_\Delta}$ till $r^{th}$ state is upper bounded by $$\cO\left(r2^n\max\left\{t_{\delta_{mut}}, t_{c}2^{n}, t_{\Delta}2^{n}, t_{=}r^32^{3n+1}\right\}\right)$$
\end{cor} 
\begin{proof}
Given the recursive construction of the relation $\mathbf{Parent_\Delta}$, we will recursively construct the relation $\mathbf{Parent^{min}_\Delta}$ and estimate the time complexity for it. Let $P^{(r)}$ denote the set of newly added entity pairs $(p, c)$ in $\mathbf{Parent_\Delta}$ in the $r^{th}$ state. By definition, $(p, c) \in P^{(r)} \Rightarrow [p \in E^{(r-1)}] \wedge [c \in E_r]$. Next for each $(p, c) \in P^{(r)}$, we need to consider the mutation histories of $p, c$: Let $MT^{(r)}_{(p, c)} = \left\{(p', c') \mid \left[p' \in E^{(r-1)}\right] \wedge \left[(p', p) \in \mathbf{R^+_{\delta_{mut}}}\right] \wedge [\left(c', c) \in \mathbf{R_{\delta_{mut}}}\right]\right\}$. Because of the injective nature of the relation $\R$, there can be at most $r$ pairs in $MT^{(r)}_{(p, c)}$, that is, $|MT^{(r)}_{(p, c)}| \leq r$.

In order to update the relation $\mathbf{Parent^{min}_\Delta}$, for each $(p, c) \in P^{(r)}$: If $MT^{(r)}_{(p, c)} \cap \mathbf{Parent^{min}_\Delta} = \emptyset$, $(p,c)$ is added to $\mathbf{Parent^{min}_\Delta}$, else not. Since for each $(p', c') \in MT^{(r)}_{(p, c)}$, $c' \in E_{r-1}$, $(p', c')$ needs to be searched only among those entity pairs which were added in $\mathbf{Parent^{min}_\Delta}$ in the $(r-1)^{th}$ state only and there could be at most $|E_{r-1}| = \cO(2^n)$ such pairs. Also $|P^{(r)}| \leq |E^{(r-1)}|*|E_r| \leq (r-1)2^{2n} = \cO(r2^{2n})$. Therefore, $|P^{(r)}|*|MT^{(r)}_{(p, c)}| = \cO(r2^{2n})*r = \cO(r^22^{2n})$ pairs need to be assessed for being already present in $\mathbf{Parent^{min}_\Delta}$. Therefore, total cost of adding new pairs to the relation $\mathbf{Parent^{min}_\Delta}$ in the $r^{th}$ could be 
\begin{eqnarray*}
t_=*|P^{(r)}|*|MT^{(r)}_{(p, c)}|*|E_{r-1}| \ \ &=& \ \  t_=*\cO(r^22^{2n})*\cO(2^n)\\ &=& \ \ \cO(t_=r^22^{3n})
\end{eqnarray*} Finally, constructing the relation $\mathbf{Parent^{min}_\Delta}$ till $r^{th}$ would cost time steps upper bounded by 
\begin{eqnarray*}
\cO\left(r2^n\max\left\{t_{\delta_{mut}}, t_{c}2^{n}, t_{\Delta}2^{n}, t_{=}r^32^{3n+1}\right\}\right) &+&  \cO(t_=r^22^{3n})\\ &=& \cO\left(r2^n\max\left\{t_{\delta_{mut}}, t_{c}2^{n}, t_{\Delta}2^{n}, t_{=}r^32^{3n+1}\right\}\right)
\end{eqnarray*}
\end{proof}

The case where entities do not have overlapping structures, we have the following corresponding bounds:   

\begin{cor}\label{complexity-en-epi_rep-nooverlap}
Given the sets of recognized entities in each state, additional time steps required for establishing an entity level replication where entities do not have overlapping structures is upper bounded by $$\cO\left(rn\max\left\{t_{\delta_{mut}}, t_{c}n, t_{\Delta}n, t_{=}r^3n^{3}\right\}\right)$$
\end{cor}

\begin{cor}\label{complexity:parent-delta-nooverlap}
Given the sets of recognized entities in each state, the worst case computational complexity of constructing relation $\mathbf{Parent_\Delta}$ till $r^{th}$ state where entities do not have overlapping structures is upper bounded by $$\cO\left(r n\max\left\{t_{\delta_{mut}}, t_{c}{n}, t_{\Delta}{n}, t_{=}r^3n^{3}\right\}\right)$$
\end{cor}

\begin{cor}\label{complexity:parent-delta-min-nooverlap}
Given the sets of recognized entities in each state, the worst case computational complexity of constructing relation $\mathbf{Parent^{min}_\Delta}$ till $r^{th}$ state where entities do not have overlapping structures is upper bounded by $$\cO\left(r n\max\left\{t_{\delta_{mut}}, t_{c}{n}, t_{\Delta}{n}, t_{=}r^3n^{3}\right\}\right)$$
\end{cor} 

Let us also consider the time complexity of algorithmically establishing the entity level replication in the Langton loops as proved in the lemma~\ref{lemma:ca-entity-rep}. Towards this, we will prove the following:

\begin{lemma}\label{complexity-en-loop-rep}
The worst case computational complexity of observing entity level replication in Langton loops in a CA model simulation is upper bounded by $\cO(r^4n^5\log n)$, where $r$ is the number of states in a replication cycle.
\end{lemma}
\begin{proof}
Since Langton loops in any given state $s$ consist of non overlapping sets of cells, the number of possible loops is upper bounded by $n$ or in other terms $|E_s| \leq n$. Also size of each loop can at most be $n$. To estimate the bounds on the time required for constructing the sets $\R$ and $C$, we can use the following algorithmic scheme: Relations $\R$ and $C$ are constructed recursively in each observed state $s$ considering the entity sets $E_{s-1}$ and $E_s$. For each entity $e = (z, pivot(z))$ in $E_{s-1}$ and $E_s$, the minimum and maximum of $x$ and $y$ coordinates are estimated as $$\left( \left[ \min\{co_x^+(z)\}, \max\{co_x^+(z)\} \right], \left[\min\{co_y^+(z)\}, \max\{co_y^+(z)\}\right]\right)$$ Using these coordinate bounds, the entities are sorted using the following ordering relation: $$(x', y') < (x'', y'') \Leftrightarrow \left[x' < x'' \right] \vee \left[(x' = x'') \wedge (y' < y'')\right]$$ This scheme can be computationally executed in at most $\cO(n\log n)$ time steps. Further, using such an ordering of the entities in $E_{s-1}$ and $E_s$, for all $(e, e') \in E_{s} \times E_{s-1}$ determining that $(e, e') \in \R$ or $(e, e') \in C$ can be performed in $\cO(n\log n)$ steps, where $\log n$ factor comes from the fact that a binary search scheme can be used to determine whether an entity $e \in E_{s}$ has the same pivot as any other entity $e' \in E_{s-1}$ (for membership in $\R$) and whether there may potentially exist any entity $e' \in E_{s-1}$, which may contain $e$. Also  in $\cO(n \log n)$ steps we can determine the subset constraints specified in the definition of $C$ for an entity pair $(e, e')$. Therefore we have, $t_{\delta\_mut} = \cO(n \log n)$ and $t_{c} = \cO(n \log n)$.  

Because mutational bounds $\delta_{mut}$ and $\delta_{rep\_mut}$ as defined earlier constrain that an entity $e'$ in $E_{s}$ can be related to only one entity $e \in E_{s-1}$, at most $n$ entity pairs $(e, e')$ could be added to relations $\R$ and $C$. Also similar to determining whether an entity $e$ is included in another entity $e'$ as demanded by the definition of $C$ discussed before, determining $(e, e') \in \Delta$ would require to establish that both $e$ and $e'$ are identical, and would take at most $\cO(n\log n)$ steps, that is, $t_{\Delta} = \cO(n \log n)$. 

The computationally most expensive process is recursively building the relation $\ancof$. As discussed before in the proof of the theorem~\ref{complexity-en-epi_rep}, using the breadth-first search based diagraph traversal algorithm, the transitive closure $(C \cup \R)^+$ in the state $s$ can be computed in $\cO(s^2n^2)$ steps because in the state $s$, relation $(C \cup \R)$ would contain at most $2sn$ pairs and total size of the entity sets $E_0, E_1, \ldots, E_{s}$ is $\leq (s+1)n$. 

Also for building the relation $\Delta$ recursively after each observed state, it would require to determine all those entity pairs $(e, e')$ such that $e \in \bigcup_{r = 0}^{s-1} E_r$ and $e' \in E_{s}$ and $e[2] = e'[2]$, that is, $e, e'$ are geometrically identical. This gives the computational upper bound of $\cO(n\log n *n * sn) = \cO(sn^3\log n)$ steps. 

Using the set intersection algorithms computing relation $\left((C \cup \R)^+ \cap \Delta \right)$ would take at most 
\begin{eqnarray*}
\cO(|(C \cup \R)^+|*|\Delta|* n \log n) &=& \cO(s^2n^2*s^2n^2* n \log n)\\ &=& \cO(s^4n^5\log n)  
\end{eqnarray*} 
steps, where again $\cO(n\log n)$ is the number of steps required for computing equality check between any two entity pairs $(e_1, e_2) \in (C \cup \R)^+$ and $(e'_1, e'_2) \in \Delta$ and the size of the transitive closure $(C \cup \R)^+$ in the state $s$ is upper bounded by $\frac{1}{2}{{sn}\choose{2}} = \cO(s^2n^2)$. 

Finally, the very first state, say the state $r$, where the relation $\ancof$ would be non empty would be the state where relation $\mathbf{Parent_\Delta}$, by its very definition, would also become non empty and this non emptiness check can be performed in constant time $\cO(1)$. 

Therefore, we can conclude that the computational complexity of the overall process of establishing the entity level reproduction in Langton loops would be upper bounded by $$\cO(n\log n) + \cO(n\log n) + \cO(r^2n^2) + \cO(rn^3\log n) + \cO(r^4n^5\log n) + \cO(1) = \cO(r^4n^5\log n)$$ time steps. 

\end{proof}

\subsubsection{Computational Complexity of Observing Fecundity} 

In order to establish fecundity having recognized an entity level reproduction, the most difficult problem for an observation process is to determine the temporal granularities for the generations of the reproducing entities especially when there may exist different types of reproducing entities with different rates of reproduction. In that case, first difficulty arises in determining how many entity types need be considered. In general this process may involve {\it assume-progress-backtrack} way of execution, where the observation process would initially scan a constant number of states to collect all different kinds of reproducing entities together with their rates of reproductions. Based upon the initial estimates on these differing rates of reproductions, it can consider their least common multiple (lcm) as the granularity for a generation and ignore other new types of entities while aiming to establish the fecundity axiom. However in case such initial estimates do not yield sufficient support for the fecundity and more reproducing entity types need to be considered, backtrack step is necessary. This process need to continue till statistically significant number of states have been observed to get support for  the fecundity axiom or to (assume it to be statistically) falsify it in that simulation. 

Let us first consider the case of single state reproduction without any epigenetic developments. In this case, computational complexity is estimated in the following result: 
\begin{theorem}\label{complexity-entity-fec-i}
Given the sets of entities in each state, the worst case computational complexity of observing fecundity without epigenetic development in an AES is upper bounded by $$\cO(L2^{\mathit{2n}}\max\{t_c, t_{\Delta}, t_{\delta_{mut}}, L/{2^{\mathit{2n}}}\})$$ where $L$ is the observed number of generations of the reproducing entities.
\end{theorem} 
\begin{proof}
Since entities would be observed as reproducing in single states, each generation would also get limited to a single observed state. Therefore, let us assume that an observation process considers $G_1, G_2, \ldots G_L$ as the temporally ordered generations of reproducing entities ranging over a state subsequence $\lan s_m, s_{m+1}, \ldots,  s_{m+L}\ran$ of size $L$. Since fecundity axiom demands counting all the reproducing entities and their progenies in these generations, it would in turn require the observation process to identify and count the reproducing instances in all the states in the state subsequence. Extending the argument as discussed before for theorem~\ref{complexity-entity-rep}, the process of enumerating the reproducing entities in each state $s$ would take at most $|E_s|*|E_{s+1}|*(t_c + t_{\Delta} + t_{\delta_{mut}})$ steps. Therefore counting the population sizes in all these generations would cost at most $L*|E_s|*|E_{s+1}|*(t_c + t_{\Delta} + t_{\delta_{mut}})$ steps. After that establishing that $\forall G_i. \exists G_{j > i}\ \textit{s.t.}\ |G_j| \geq |G_i|$ would require additional $\cO(L^2)$ steps. Thus overall time complexity of observing fecundity is no less than 
\begin{eqnarray*}
L*|E_s|*|E_{s+1}|*(t_c + t_{\Delta} + t_{\delta_{mut}}) + \cO(L^2) 
& \leq & L*2^n*2^n*(t_c + t_{\Delta} + t_{\delta_{mut}}) + \cO(L^2) \\
&=& \cO(L2^{\mathit{2n}}\max\{t_c, t_{\Delta}, t_{\delta_{mut}}, L/{2^{\mathit{2n}}}\})
\end{eqnarray*}
steps. 
\end{proof}
Next, we consider the more general case involving epigenetic developments in the child entities:
\begin{theorem}\label{complexity-en-fec-ii}
Given the sets of entities in each state, the worst case computational complexity of observing fecundity is upper bounded by $\cO\left(L\max\{t_{\delta_{mut}}2^n, t_{c}2^{2n}, t_{\Delta}r_{\pi}2^{2n}, t_{=}r_{\pi}^42^{4n}, L\}\right)$, where where $L$ is the observed number of generations of the reproducing entities and $r_{\pi}$ is the maximum of the lengths of the reproduction cycles of the different types of observed reproducing entities across these generations.
\end{theorem} 
\begin{proof}
Similar to the proof above, let us assume that an observation process considers $G_1, G_2, \ldots G_L$ as the temporally ordered generations of reproducing entities ranging over a state subsequence $\lan s_m,$ $s_{m+1}, \ldots$, $s_{m+r}\ran$ of size $r$. As per the discussion before, let there be $k$ different types of reproducing entities under observation in each generation with possibly different number of states in their reproduction cycles $r_1, r_2, \ldots, r_k$. Let $\lambda = \operatorname{lcm}(r_1, r_2, \ldots, r_k)$, $r_\delta = \min\{r_i\}$, and $r_\pi = \max\{r_i\}$. 

Therefore, we have $r = \lambda L$, and each generation would have a granularity of $\lambda$ number of states. Since fecundity axiom demands counting all the reproducing entities and their progenies in these generations, it would in turn require the observer to identify and count the reproducing instances in each generation. 

As a corollary of the theorem~\ref{complexity-en-epi_rep}, establishing that an entity is an instance of reproductive type requiring $r_i$ states in its reproduction cycle would take $\zeta_i = \cO\left(\max\left\{t_{\delta_{mut}}2^n, t_{c}2^{2n}, t_{\Delta}r_i2^{2n}, t_{=}r_i^42^{4n}\right\}\right)$ steps assuming that the values for $t_{\delta_{mut}}, t_{c}, t_{\Delta}$,and  $t_{=}$ are comparable for all the entity types. Therefore counting the number of reproducing entities for all these different reproducing entities types in each generation would in turn amount to recognizing the reproductive instances in the states $s_m, s_{m+1}, \ldots s_{m+\lambda-r_{\delta}}$ in the state $G_1$ and so on. The corresponding time steps required are upper bounded by
\begin{eqnarray*}\displaystyle \sum_{1\leq i \leq k}(\lambda-r_{\delta})\zeta_i &=& \cO(\max_{1\leq i \leq k}\{\zeta_i\}) \\
&=& \cO(\lambda-r_{\delta})*\cO\left(\zeta_{\max\{r_i\}}\right) \\
&=& \cO\left(\lambda\max\left\{t_{\delta_{mut}}2^n, t_{c}2^{2n}, t_{\Delta}\max\{r_i\}2^{2n}, t_{=}(\max\{r_i\})^42^{4n}\right\}\right) \\
&=& \cO\left(\lambda\max\left\{t_{\delta_{mut}}2^n, t_{c}2^{2n}, t_{\Delta}r_\pi2^{2n}, t_{=}r_\pi^42^{4n}\right\}\right)
\end{eqnarray*} 
steps across the starting states in the subsequence ranging over these generations of entities. This gives the upper bound for counting the population sizes for all these generations as $$\Theta = L*\cO\left(\lambda\max\left\{t_{\delta_{mut}}2^n, t_{c}2^{2n}, t_{\Delta}r_\pi2^{2n}, t_{=}r_\pi^42^{4n}\right\}\right)$$ After that establishing that $\forall G_i. \exists G_{j > i}\ \textit{s.t.}\ |G_j| \geq |G_i|$ would require additional $\cO(L^2)$ steps. Thus overall time complexity of observing fecundity is no less than $$\Theta + \cO(L^2) = \cO\left(r2^n\max\left\{t_{\delta_{mut}}, t_{c}2^{n}, t_{\Delta}r_\pi2^{n}, t_{=}r_\pi^42^{3n}\right\}\right)$$ steps. 
\end{proof}
A special case of replication (with epigenetic development) involving no reproductive mutations in the child entities and no parental mutations as well would only demand identification using syntactic equivalence or tags and counting the entities belonging to various reproductive types only in last state of each generation, that is, in the states $s_{m+\lambda}, s_{m+2\lambda}, \ldots, s_{m+L\lambda}$. This process may cost at most 
\begin{eqnarray*} 
\displaystyle \sum_{1\leq i \leq L}\left(|E_{m+i\lambda}|*k*t_=\right) &\leq & L*2^n*2^n*t_= \\ &=& \cO(Lt_=2^{2n})
\end{eqnarray*}
steps.

Also the case where entities do not have overlapping structures, we have the following corresponding bounds:   
    
\begin{cor}\label{complexity-entity-fec-i-nooverlap}
Given the sets of recognized entities in each state, the additional time steps required for observing fecundity without epigenetic development in an AES where entities do not have overlapping structures is upper bounded by $$\cO(L n^{2}\max\{t_c, t_{\Delta}, t_{\delta_{mut}}, L/n^{2}\})$$
\end{cor} 
\begin{cor}\label{complexity-en-fec-ii-nooverlap}
Given the sets of recognized entities in each state, the additional time steps required for observing fecundity where entities do not have overlapping structures is upper bounded by $$\cO\left(L n\max\{t_{\delta_{mut}}, t_{c}n, t_{\Delta}r_{\pi}n, t_{=}r_{\pi}^4n^{3}, L\}\right)$$
\end{cor} 
    
Let us next consider the time complexity of algorithmically establishing the population level replication in the Langton loops as proved in the lemma~\ref{lemma:ca-entity-rep}. Towards this, we will prove the following:

\begin{lemma}\label{complexity-en-loop-fec}
The worst case computational complexity of observing fecundity in Langton loops in a CA model simulation is upper bounded by $\cO(\max\{Lr^4n^5\log n, L^2\})$, where $r$ is the number of states in a replication cycle.
\end{lemma}
\begin{proof}
As is the case with proof of lemma~\ref{lemma:ca-entity-rep}, a natural way an observation process could infer each generation $G_i$ is with a granularity of $r$ since after each $r$ states all those loops with open neighborhood would be able to replicate. Thus the temporally ordered generations of entities $G_1, G_2, \ldots G_L$ would range over the state subsequence $\lan s_0, s_1, \ldots,s_r, s_{r+1}, \ldots, s_{2r},\ldots,  s_{Lr}\ran$. Utilizing the recursive scheme for observing entity level reproduction, the counting of reproducing entities can also be carried out only at the intervals of $r$ states. As discussed before, such counting itself would take at most $\cO(n)$ steps (together with the cost of establishing the entity level reproduction estimated before) in any state as a result of the process of entity recognition itself. Therefore recognizing and counting the reproducing entities in all these generations would take at most 
\begin{equation*}
\begin{split}
\cO(\# generations * \# \mbox{entities in each generation} &* \mbox{cost of recognizing each reproducing entity})\\ &= \cO(L*n*r^4n^5\log n)\\ &= \cO(Lr^4n^5\log n)
\end{split}
\end{equation*} 
steps, where the number reproducing entities in each state are bounded by $n$. Having counted the population sizes in each generation, establishing the non decreasing property for these population sizes as demanded by the axiom would take additional $\cO(L^2)$ steps. Hence overall computational complexity of establishing fecundity in Langton's model would be upper bounded by $$\cO(Lr^4n^5\log n + L^2) = \cO(\max\{Lr^4n^5\log n, L^2\})$$ 
\end{proof}

\subsubsection{Computational Complexity of Observing Heredity} 

\begin{theorem}\label{complexity-heredity}
Given the sets of recognized entities in each state, the worst case computational complexity of observing heredity in an AES is upper bounded by  $$\cO\left(r2^n\max\left\{t_{\delta_{mut}}, t_{c}2^{n}, t_{\Delta}2^{n}, t_{=}r^32^{3n+1}, |\Upsilon|^2t_d2^{n}\right\}\right)$$
\end{theorem} 
\begin{proof}
Consider the state subsequence $\Omega = \langle s_m, \ldots s_r \rangle , m \ll r$ as defined in the heredity axiom~\ref{axiom:heredity}. Following the proof of corollary~\ref{complexity:parent-delta-min}, we know that the process of constructing the relation $\mathbf{Parent^{min}_\Delta}$ may take at most $\Pi^{(r)} =  \cO\left(r2^n\max\left\{t_{\delta_{mut}}, t_{c}2^{n}, t_{\Delta}2^{n}, t_{=}r^32^{3n+1}\right\}\right)$ time steps till the state $s_r$. 

Therefore, the process of constructing the set $\mathbf{Parent_{\Delta}^{\Omega}}$ would also have time complexity upper bounded by $\Pi$. Furthermore, size of the set $\mathbf{Parent_{\Delta}^{\Omega}}$ is also upper bounded by $\cO(r2^{2n})$. Hence determining the set $\mathbf{Inherited^{i}_{\Omega}}$ may require at most  $X^{(r)}_i = |\mathbf{Parent_{\Delta}^{\Omega}}|*|\Upsilon|*t_{D} = \cO(|\Upsilon|rt_{D}2^{2n})$ steps where $t_d$ is the expected time required to compare $D(e, e')[i]$ with $0_{\mathit{diff_i}}$ to check the equality (or inequality) as demanded in the definition of $\mathbf{Inherited^{i}_{\Omega}}$. 

Therefore establishing that there exists a characteristics $\mathit{char}_i \in \Upsilon$ for which $\displaystyle \lim_{r \rightarrow \infty}|\mathbf{Parent_{\Delta}^{\Omega}}|$ $/|\mathbf{Inherited_{\Omega}^i}|$ $\leadsto 1$ may in turn could require at most
\begin{eqnarray*}
\Pi^{(r)} + \sum_{char_i \in \Upsilon}X^{(r)}_i &\leq& \Pi^{(r)} + |\Upsilon|\max\{X^{(r)}_i\}\\ 
&=& \cO\left(r2^n\max\left\{t_{\delta_{mut}}, t_{c}2^{n}, t_{\Delta}2^{n}, t_{=}r^32^{3n+1}\right\}\right) + \cO(|\Upsilon|^2rt_d2^{2n}) \\
&=& \cO\left(r2^n\max\left\{t_{\delta_{mut}}, t_{c}2^{n}, t_{\Delta}2^{n}, t_{=}r^32^{3n+1}, |\Upsilon|^2t_d2^{n}\right\}\right)
\end{eqnarray*} steps.
\end{proof}

The case where entities do not have overlapping structures, we have the following corresponding bound:   

\begin{cor}\label{complexity-heredity-nooverlap}
Given the sets of recognized entities in each state, the worst case computational complexity of observing heredity in an AES where entities do not have overlapping structures is upper bounded by  $$\cO\left(r n\max\left\{t_{\delta_{mut}}, t_{c}{n}, t_{\Delta}{n}, t_{=}r^3n^{3}, |\Upsilon|^2t_d{n}\right\}\right)$$
\end{cor} 
    
\subsubsection{Computational Complexity of Observing Natural Selection} 

\begin{lemma}\label{complexity-ns1}
Given the sets of recognized entities in each state and the relations $\RR$ and $\p$ from the earlier steps, additional time steps required for constructing the set $\mathbf{\Lambda_{min}}$ as defined in the axiom~\ref{axiom-ns1} is upper bounded by $\cO\left(t_=r^32^{3n}\right)$.
\end{lemma} 
\begin{proof}
Given the relation $\p$, constructing sets $SR(s_j)$ requires that for each $p \in E_{s_j}$, if $\exists \ (p, c) \in \p$, then $p$ is included in $SR(s_j)$ else not. This process may take at most 
\begin{eqnarray}\label{eqPi}
\Pi_{s_j} &=& |E_{s_j}| * (|\p| * t_=)  \nonumber \\ 
&=& \cO(2^n * r^22^{2n} * t_=) \nonumber \\ 
&=& \cO(t_=r^22^{3n}) 
\end{eqnarray} steps. Therefore construction of all these sets and their union $\mathbf{\Lambda} = \bigcup_{s_j \in \Omega} SR(s_j)$ may take at most $c_{\Omega} = \sum_{s_j \in \Omega}\Pi_{s_j} \approx r * \Pi_{s_j} = \cO(t_=r^32^{3n})$ steps.

Next, let us consider a process of constructing $\mathbf{\Lambda_{min}}$. For each $p \in SR(s_j)$, consider its mutation sequence $(p, p_1), (p, p_2), \ldots$ as appearing in $\RR$: There would be at most $(r - j)$ such pairs in such a sequence and identifying these pairs by tracing all outgoing edges from node $p$ in the graph of $\RR$ would take at most $\cO(r)$ steps. Next we delete each of $p_1, p_2, \ldots$ from $\mathbf{\Lambda}$ keeping only $p$. Such a deletion may take at most $(r-j)*|\mathbf{\Lambda}|*t_=$ steps. This yields the overall cost of identification and deletion of reproducing mutants of $p$ as   
\begin{eqnarray*}
id_p = \cO(r) + (r-j)*|\mathbf{\Lambda}|*t_= &=& \cO(r * r^22^{2n} * t_=) \\ &=& \cO(t_=r^32^{2n}) 
\end{eqnarray*} steps. Such identification and deletion of the mutants need to repeated for each $p \in SR(s_j)$, which may require at most 
\begin{eqnarray*}
id_{s_j} &=& \sum_{p \in SR(s_j)} id_p \\ & \approx & |SR(s_j)| * id_p \\ &=& \cO(2^n) * \cO(t_=r^32^{2n}) \\
&=& \cO(t_=r^32^{3n})
\end{eqnarray*} steps. Therefore constructing $\mathbf{\Lambda_{min}}$ from given $\mathbf{\Lambda}$ and $\RR$ would require at most $id_{\Omega} = \sum_{s_j \in \Omega} id_{s_j} = \cO(t_=r^42^{3n})$ steps. In turn, given $\p$ and $\RR$, constructing $\mathbf{\Lambda_{min}}$ would require 
\begin{equation}\label{ns1-eq1}
{\Delta}_\Omega = c_{\Omega} + id_{\Omega} = \cO(t_=r^32^{3n})
\end{equation} steps.
\end{proof}

\begin{lemma}\label{complexity-ns2}
Given the relations $\RR$, $\p$, and $\mathbf{\Lambda_{min}}$ from the earlier steps, additional time steps required for establishing the axiom of sorting~\ref{axiom-ns2} are upper bounded by $\cO(r2^n\max\{r2^n, |\Upsilon|\})$. 
\end{lemma} 
\begin{proof}
For the axiom of Sorting, one needs to estimate the size of $\mathbf{\Lambda_{min}}$, which is upper bounded by the size of $E_{\Omega}$, which at worst could be as large as $\cO(|\Omega|*2^n) = \cO(r2^n)$. Next, for a given $p \in \mathbf{\Lambda_{min}}$, the set $Child_p$ can be constructed as follows: Consider the mutation sequence $(p, p_1), (p, p_2), \ldots$ for $p$ in $\RR$ of size at most $(r - j) \leq r$ where $p \in E_{s_j}$. For each of $p, p_1, \ldots$ their children need to be counted in $\p$. In the worst case of group reproduction where a large fraction of the entities in a state collectively reproduce and give rise to child entities which also form a large fraction  of all the entities in their respective state (i.e., of the order of $\cO(2^n)$), determining $Child_p$ may in turn  be bit costly and may take time steps of the order of $\cO(r2^n)$. Therefore given the sets $\RR$ and $\p$, estimating the number of child entities and in turn determining the rate of reproduction for all the parent entities $p$ in $\mathbf{\Lambda_{min}}$ could take in the worst case at most $= \sum_{p \in \mathbf{\Lambda_{min}}}\cO(r2^n) = |\mathbf{\Lambda_{min}}|* \cO(r2^n) = \cO(r^22^{2n})$ steps. 

For the calculation of variances in the characteristics in $\Upsilon$ and in the rates of reproduction for entities in $\mathbf{\Lambda_{min}}$, a recursive reformulation of the sample mean and sample variance could be used together with the above discussed process of estimating $ror$ for entities in $\mathbf{\Lambda_{min}}$. 

For a recursively generated stream of data: $d_1, d_2, \ldots, d_t$, let $\mu_{t}, {\sigma}^2_{t}$ be the sample mean and variance respectively till time point $t$ and let $d_{t+1}$ be the newly generated data at time point $t+1$. The updated sample mean $\mu_{t+1}$ and variance ${\sigma}^2_{t+1}$ could be calculated as follows:    
\begin{eqnarray*}
{\sigma}^2_{t+1} &=& \frac{t}{t+1}{\sigma}^2_{t} + \frac{1}{t}(d_{t+1} - \mu_{t+1})^2 \mbox{ where,} \\
\mu_{t+1} &=&  \frac{t}{t+1}\mu_{t} + \frac{1}{t+1}d_{t+1} \mbox{ with} \\
\mu_1 &=& d_1 \mbox{ and} \\
{\sigma}^2_{1} &=& 0
\end{eqnarray*}

Calculation based upon the above reformulation of mean and variance would demand additional time steps of the order of $\cO(|\Upsilon|)$ for estimating variance in the characteristics in $\Upsilon$ and a constant number of steps ($\cO(1)$) for estimating variance in $ror$. Therefore given the sets $\RR$, $\p$, and $\mathbf{\Lambda_{min}}$ from the earlier steps, additional time steps required for establishing the axiom of sorting are upper bounded by
\begin{eqnarray}\label{eqn-ns2}
|\mathbf{\Lambda_{min}}|* \left[\cO(r2^n) + \cO(|\Upsilon|) + \cO(1) \right] &=&  \cO(r2^n) * \cO(\max\{r2^n, |\Upsilon|\})\nonumber \\ 
&=& \cO(r2^n\max\{r2^n, |\Upsilon|\}) 
\end{eqnarray}
\end{proof}

\begin{lemma}\label{complexity-ns3}
Given the sets $\p$ and $\mathbf{\Lambda_{min}}$ from the earlier steps, additional time steps required for establishing the axiom~\ref{axiom-ns3} of Heritable Variation is upper bounded by $\cO(r2^{2n}\max\{r^32^{2n}, t_d|\Upsilon|\})$. 
\end{lemma} 
\begin{proof}
Given the sets $\p$ and $\mathbf{\Lambda_{min}}$ from the earlier steps, additional time steps required for constructing the set $\mathbf{Ch_{mut}}$ as defined in the axiom~\ref{axiom-ns3} is upper bounded by 
\begin{eqnarray}\label{eqn-ns3-1}
|\mathbf{\Lambda_{min}}|*\cO(2^n)*|\Upsilon|* t_d &\leq& |E_{\Omega}|*\cO(2^n)*|\Upsilon|* t_d \nonumber \\ 
&=& \cO(r2^n)*\cO(2^n)*|\Upsilon|* t_d \nonumber \\ 
&=& \cO(t_dr|\Upsilon|2^{2n})
\end{eqnarray}
Next, in order to determine the size of the largest subset $\mathbf{Var\_Ch_{mut}}$ of $\mathbf{Ch_{mut}}$, consisting of only those entities which are unique w.r.t. at least one characteristic, an intermediate relation $Ch_{\mathit{diff}} \subseteq \mathbf{Ch_{mut}} \times \mathbf{Ch_{mut}}$ is constructed such that $$(c, c') \in Ch_{\mathit{diff}} \Leftrightarrow \exists \mathit{char}_i \in \Upsilon \ \mbox{s.t.}\  0_{\mathit{diff}_i} \prec D(c, c')[i]$$ $Ch_{\mathit{diff}}$ essentially consists of those entity pairs, which differ with each other at least in one characteristic. Now, by definition, the set $\mathbf{Var\_Ch_{mut}}$ corresponds the maximum clique in the graph of $Ch_{\mathit{diff}}$. However, the problem of fining the maximum clique (MAXCLIQUE) in a graph is known to be a NP complete problem~\cite{maxclique99} and therefore an exact algorithm would currently require $\cO(2^{\epsilon a})$ steps where $a$ is the number of nodes in the graph and $\epsilon $ is some constant $> 0$. However, if we do not demand the exact solution, there exist polynomial time approximation algorithms, which achieve the approximation ratio of $a^{1 - o(n)}$. Using one such algorithm presented in~\cite{maxclique96}, the set $\mathbf{Var\_Ch_{mut}}$ can be constructed (with the approximation of $2^{(\log{r} + n)(1-o(n))}$ in the size of the estimated subset) requiring time steps upper bounded by $\cO(a^4) = \cO(r^42^{4n})$. 

Having constructed the set $\mathbf{Var\_Ch_{mut}}$, determining that $\displaystyle \lim_{r \rightarrow \infty}\frac{|\mathbf{Var\_Ch_{mut}}|}{|\mathbf{\Lambda_{min}}|} \leadsto 1$ would take only constant number of steps. 

Therefore given the sets $\p$ and $\mathbf{\Lambda_{min}}$ from the earlier steps, establishing the axiom of heritable variation would take additional time steps upper bounded by  
\begin{eqnarray}\label{eqn-ns3-2}
\cO(t_dr|\Upsilon|2^{2n}) + \cO(r^42^{4n}) = \cO(r2^{2n}\max\{t_d|\Upsilon|, r^32^{2n}\})
\end{eqnarray}
\end{proof}

\begin{lemma}\label{complexity-ns4} 
Given the set $\mathbf{Var\_Ch_{mut}}$, and the function $\mathbf{ror}$ from earlier steps, the axiom~\ref{axiom-ns4} of correlation would require additional time steps upper bounded by $\cO(r|\Upsilon|2^{n})$. 
\end{lemma} 
\begin{proof}
Given the set $\mathbf{Var\_Ch_{mut}}$, and the function $\mathbf{ror}$ from the earlier steps, the axiom of correlation only requires computing various parameters: $\mu_{ror}, \mu_{i}$, and $r_{i}$ for each characteristic $\mathit{char}_i \in \Upsilon$. Assuming that numerical addition and division of two numbers take only constant number of time steps, the computation of $\mu_{ror}$ and $\mu_{i}$ is bounded by the size of the set $\mathbf{Var\_Ch_{mut}}$, which could at most be $\cO(r2^n)$. Having computed $\mu_{ror}$ and $\mu_{i}$, the computation of $r_i$ would take additional time steps bounded by the size of the set $\mathbf{Var\_Ch_{mut}}$ assuming that multiplication and square-root operations also take constant number of time steps. Therefore establishing the fact that $\exists char_i \in \Upsilon$ s.t. $r_i > 0$, would take time steps upper bounded by
\begin{eqnarray}\label{eqn-ns4-1}
|\Upsilon| * \left[\cO(r2^{n}) + \cO(r2^n)\right] = \cO(r|\Upsilon|2^{n})
\end{eqnarray}
\end{proof}
\begin{theorem}\label{complexity-entity-ns}
Given the sets of recognized entities in each state and the relations $\RR$ and $\p$, additional time steps required for establishing natural selection in an AES are upper bounded by $$\cO\left(r2^{2n}\max\left\{t_=r^22^n, t_d|\Upsilon|, r^32^{2n}\right\}\right)$$
\end{theorem} 
\begin{proof}
Establishing the case of natural selection would require an observation process to establish all the four axioms (\ref{axiom-ns1}, \ref{axiom-ns2}, \ref{axiom-ns3}, and  \ref{axiom-ns4}) specified for it. Using the upper bounds established above for these axioms in the equations ~\ref{ns1-eq1}, \ref{eqn-ns2}, \ref{eqn-ns3-2}, and  \ref{eqn-ns4-1}, the time steps required in the worst case are 
\begin{eqnarray}\label{eqn_ns} 
&& \cO\left(t_=r^32^{3n}\right) + \cO(r2^n\max\{r2^n, |\Upsilon|\}) + \nonumber \\
&& \cO(r2^{2n}\max\{t_d|\Upsilon|, r^32^{2n}\}) + \cO(r|\Upsilon|2^{n}) \nonumber \\
&=& \cO\left(r2^{2n}\max\left\{t_=r^22^n, t_d|\Upsilon|, r^32^{2n}\right\}\right)
\end{eqnarray}
\end{proof}

Given the estimates for the upper bounds on the time steps required for constructing the entity sets $E_\Omega$, $\RR$, and $\p$, the bound for the overall computational complexity of observing natural selection can be estimated: 

\begin{cor}\label{complexity-entity-ns-total}
Overall worst case computational complexity of establishing natural selection in an AES is upper bounded by $$\cO\left(r2^{n}\max\left\{t_{\delta_{mut}}, t_{c}2^{n}, t_{\Delta}2^{n}, t_{=}r^32^{3n+1}, t_d|\Upsilon|2^n\right\}\right)$$
\end{cor} 

The case where entities do not have overlapping structures, we have the following corresponding bounds:   

\begin{cor}\label{complexity-ns1-nooverlap}
Given the sets of recognized entities in each state and the relations $\RR$ and $\p$ from the earlier steps, additional time steps required for constructing the set $\mathbf{\Lambda_{min}}$ as defined in the axiom~\ref{axiom-ns1} where entities do not have overlapping structures is $\cO\left(r n\max\left\{t_{\delta_{mut}}, t_{c}{n}, t_{\Delta}n, t_{=}r^3n^{3}\right\}\right)$.
\end{cor} 

\begin{cor}\label{complexity-ns2-nooverlap}
Given the sets $\RR$, $\p$, and $\mathbf{\Lambda_{min}}$ from the earlier steps, additional time steps required for establishing the axiom of sorting~\ref{axiom-ns2} where entities do not have overlapping structures are upper bounded by $\cO(r n\max\{r n, |\Upsilon|\})$. 
\end{cor} 

\begin{cor}\label{complexity-ns3-nooverlap}
Given the sets $\p$ and $\mathbf{\Lambda_{min}}$ from the earlier steps, additional time steps required for establishing the axiom~\ref{axiom-ns3} of Heritable Variation is upper bounded by $\cO(r n^{2}\max\{|\Upsilon|, r^3n^{2}\})$. 
\end{cor}

\begin{cor}\label{complexity-ns4-nooverlap} 
Given the set $\mathbf{Var\_Ch_{mut}}$, and the function $\mathbf{ror}$ from earlier steps, the axiom~\ref{axiom-ns4} of correlation would require additional time steps upper bounded by $\cO(r|\Upsilon|{n})$. 
\end{cor} 

Finally,

\begin{cor}\label{complexity-entity-ns-nooverlap}
Given the sets of recognized entities in each state and the relations $\RR$ and $\p$, additional time steps required for establishing natural selection in an AES where entities do not have overlapping structures are upper bounded by $$\cO\left(rn^{2}\max\left\{t_=r^2n, t_d|\Upsilon|, r^3n^{2}\right\}\right)$$
\end{cor} 

And,

\begin{cor}\label{complexity-entity-ns-total-nooverlap}
Overall worst case computational complexity of establishing natural selection in an AES where entities do not have overlapping structures is upper bounded by $$\cO\left(r{n}\max\left\{2^n, t_{\delta_{mut}}, t_{c}{n}, t_{\Delta}{n}, t_{=}r^3n^{3}, t_d|\Upsilon|n\right\}\right)$$
\end{cor} 

\subsection*{Summary of Results}

The following table~\ref{tab:Summary} summarizes various computational complexity bounds established in this section:

\begin{table}[htbp]
	\centering
{
\begin{tabular}[b]{||p{3.75cm}|p{5.25cm}|p{5.1cm}||}
\hline
\hline
& \multicolumn{2}{|c|}{\textbf{Upper Bounds on the Worst Case Computational Complexity}} \\ \cline{2-3}
\textbf{Evolutionary Components} & \textbf{Entities with Overlapping Structures} & \textbf{Entities with non-Overlapping Structures} \\ \hline

Entity Recognition & (NP-Hard) $\cO(2^n)$ & $\cO(n2^n)$ \\ \hline

Entity Level Reproduction without Epigenetic Development & $\cO(\eta 2^n \max\{t_c, t_{\Delta}, t_{\delta_{mut}}2^{n}\})^\#$ & $\cO(\eta n\max\{t_c, t_{\Delta}, t_{\delta_{mut}}n\})^\#$ \\ \hline 

Entity Level Reproduction with Epigenetic Development & $\cO(r2^n\max\{t_{\delta_{mut}}, t_{c}2^{n}, t_{\Delta}2^{n}$, $t_{=}r^32^{3n}\})^\#$ & $\cO(rn\max\{t_{\delta_{mut}}, t_{c}n, t_{\Delta}n, t_{=}r^3n^{3}\})^\#$ \\ \hline 

Fecundity without Epigenetic Development & $\cO(L2^{\mathit{2n}}\max\{t_c, t_{\Delta}, t_{\delta_{mut}}, L/{2^{\mathit{2n}}}\})^\#$ & $\cO(L n^{2}\max\{t_c, t_{\Delta}, t_{\delta_{mut}}, L/n^{2}\})^\#$\\ \hline 

Fecundity with Epigenetic Development & $\cO(L\max\{t_{\delta_{mut}}2^n, t_{c}2^{2n}, t_{\Delta}r_{\pi}2^{2n}$, $t_{=}r_{\pi}^42^{4n}, L\})^\#$ & $\cO(L n\max\{t_{\delta_{mut}}, t_{c}n, t_{\Delta}r_{\pi}n$, $t_{=}r_{\pi}^4n^{3}, L\})^\#$ \\ \hline 

Heredity & $\cO(r2^n\max\{t_{\delta_{mut}}, t_{c}2^{n}, t_{\Delta}2^{n}$, $t_{=}r^32^{3n+1}, |\Upsilon|^2t_d2^{n}\})$ & $\cO(r n\max\{t_{\delta_{mut}}, t_{c}{n}, t_{\Delta}{n}$, $t_{=}r^3n^{3}, |\Upsilon|^2t_d{n}\})$ \\ \hline

Natural Selection & $\cO(r2^{n}\max\{t_{\delta_{mut}}, t_{c}2^{n}, t_{\Delta}2^{n}$, $t_{=}r^32^{3n+1}, t_d|\Upsilon|2^n\})$ & $\cO(r{n}\max\{2^n, t_{\delta_{mut}}, t_{c}{n}, t_{\Delta}{n}$, $t_{=}r^3n^{3}, t_d|\Upsilon|n\})$ \\ \hline
\multicolumn{3}{|l|}{{\it $^\#$: Given the sets of recognized entities in the states.}} \\ \hline
\hline
\end{tabular}
}
	\caption{Summary of the worst case computational complexity bounds of observing evolutionary components.}
	\label{tab:Summary}
\end{table}
 
\section{Related Work} \label{chap:related}

Because of the presence of sufficiently many biology-specific criterion and tests to distinguish life from non-life, e.g., specific bio-molecules, PCR, and fluorescent antibody markers, in biological literature there is little formal work on recognizing life {\it per se}. However, in astrobiological studies, defining criterion to detect life in an arbitrary chemical environment is still an area of active research. For example, McKay~\cite{McKay04} proposes the ``Lego Principle", according to which biological processes use only a distinctive subset of possible organic molecules, whereas, abiotic processes have relatively smooth distribution over all organic molecules.  
Also, recently, Melkikh~\cite{Alexey08} has considered the computational analogue of the problem of the origin of species in a genome space under DNA Computing framework~\cite{dna_computing} and has shown that in absence of a priori information about the possible species of organisms, the underlying computational problem is NP hard and therefore cannot be solved in polynomial number of steps.   

To the author's knowledge, there is not much work focusing on the observation process for AES studies reported in literature. Though there exist proposals to define `numerical parameters' or `statistics' \cite{Bedau99} to recognize life in a model. However, it is not clear whether there can be simple numerical definitions capturing the essence of life in arbitrary models and even if so does not seem to be the case with the existing proposals. The difficulty arises out of intricate nature of reproduction and selection inevitably involving non trivial identification of the population of evolving entities. Langton defined in~\cite{Langton91} a quantitative metric, called \emph{lambda}
parameter to detect life in any generic one dimensional cellular automata model based upon the characteristics of its transition rules. This lambda parameter based analysis is based upon the assumption that any self organizing system can be treated as living and does not consider population centric evolutionary behavior as characteristic of life. In \cite{Bedau98} there is a discussion on the classification of long term adaptive evolutionary dynamics in natural and artificially evolving systems. This they achieve by defining activity statistics for the components, which quantifies the adaptive value of components (characteristics in our model). They employ similar mechanism as of ours by associating activity counters (tags) with all the components present in the system during simulation.


Self-reproduction, which has a long history of research starting from the late 1950s~\cite{Burks70,Sipper98,Freitas04} has evaded precise formal definition applicable to a wide range of models~\cite{ND98} in the sense of observable characterization of the reproducing entities. Though there is enough work on mathematical analysis of replication dynamics (fecundity) in various natural systems or the systems where environmental constraints governing the rate of reproductions are known (see for overview~\cite[Ch. 5]{Freitas04}.) In some of the discussions related to self-replication in cellular automata models~\cite{Samaya98b,Morita98}, formalizations of reproducing structures are presented, but they do not attempt to provide a general framework for observing reproduction or other components
of evolutionary processes. These attempts at formalizing reproduction in CA models are reminiscent of our definition of
entities (loops) in Section~\ref{sec:observe}. 

In other work~\cite{Misra06}, we proposed a multi-set theoretic framework to formalize self reproduction (with mutations) in dynamical hierarchies in terms of hierarchal multi-sets and corresponding inductively defined meta-reactions. The ``self" in ``self-reproduction" was defined in terms of {\it observed structural equivalences} between entities. We also introduced constraints to distinguish a simple ``collection" of reacting entities from genuine cases of ``emergent" organizational structures consisting of {\it semantically coupled} multi-set of
entities.


\section{Conclusion} \label{sec:concluding}

\subsection{General Remarks}


This work is an attempt to bring the implicitly assumed notion of {\sf observations to be carried out independent of the underlying structure of the model} into main focus of AES studies by placing observations into distinct formal platform. The work can also be seen as an attempt to fulfill the need for explicitly separating the design of the AES models from the abstractions used to describe their dynamic progression and the discovery of life-like evolutionary behavior. We consider evolutionary behavior, as one such characteristic property of life-like phenomena and formally capture basic components for observing evolutionary behavior in AES models.  We formally elaborate in algebraic terms necessary and sufficient steps for an observation process, to be employed by an AES researcher upon the time progressive simulations of his model universe. The observation process as specified in this work may be carried out manually or can be alternatively algorithmically programmed and integrated within the model itself.

To define an inference process we specify necessary conditions, as axioms, which must be satisfied by the outcomes of observations made upon the model universe in order to infer the extent to which evolutionary phenomena is present in the model.

Computational complexity theoretic analysis of the expert guided entity recognition as well as establishing evolutionary behavior  reveals that an automated discovery of life-like phenomena could be computationally intensive in practice and techniques from the fields of  pattern recognition and machine learning in general can be of significant 
use for such purposes. 

The framework design and the case study analysis on Langton loops and Algorithmic Chemistry also provide clues for AES research designs so that to be better able to witness evolutionary phenomena in the model during its simulations. This is discussed next:

{\sf Sufficient Reproduction with Variation:} Existence of potentially large set of reproducing entities with significant variation in their characteristics is essential for the presence of evolutionary phenomena in the model. Quite often this hinges upon the choice of reaction rules or the semantics of the model and indeed it is a serious challenge for any model designer to define the reaction semantics which permits potentially large set of reproducers with significant variation. Another interesting aspect is that these reproducers must be relatively closely related to each other under the reaction semantics. This means that sufficiently many variations of reproducers should also be reproducers in themselves otherwise the axiom of preservation of reproduction under mutation will not effectively hold in the model and most of the reproducers would have to appear \emph{de novo} during simulations. We encounter this problem in both of the case studies of Langton loops and AlChemy. In case of Langton loops, any kind of change in the loop structure would cause cessation of replication. The work on designing Evoloops is therefore based upon the redefinition of the reaction semantics or transition rules which permit variation in replicating loops. Similarly in the case of Algorithmic chemistry, almost all of the single replicating $\lambda$ terms arise \emph{de novo} and their variations do not replicate under $\beta$ reaction semantics.

{\sf Measurable Rates of Reproduction:} The model should be designed such that it is possible to impose some valid measure of determining the rates of reactions which in turn can be used to estimate differences in the rates of reproduction of different entities. This measurement of reproductive rates must be independent of the updation algorithm which selects entities for reaction. Therefore it can be argued that the models, where all (reproductive) reactions take place in a single step would be difficult to observe for natural selection, which works only when different entities reproduce at different rates.  For example, it is not possible to infer differences in the rates of reproduction  among different reproducing elementary hypercycles in the Algorithmic Chemistry consisting of the same number of $\lambda$ terms because every reaction between any two $\lambda$ terms occurs in a single step. On the other hand natural selection can be observed in case of Evoloops precisely because different types of loops consisting of different number of cells reproduce at different rates based upon the number of state transitions.

\subsection{Limitations}
\label{sec:limit}

The presented framework does not place direct emphasis on certain concepts widely associated with AES studies. In our current setting the notion of ``strong emergence'' is only implicitly present and ``the element of surprise'' \cite{Bass97} often associated with emergence is not immediate in the framework. Similarly ``the element of autonomy'' of emergent entities with respect to the underlying micro-level dynamics is not addressed in our framework. Indeed, the spirit of the high level of observations and corresponding abstractions upon which the framework rests, may preclude such inferences. Nonetheless the idea of ``weak emergence" \cite{Bedau97}, which lays emphasis on the simulations of the model for the emergence of high level macro-states is fundamental to our framework, where the observation process is by default based upon the simulations of the model and not on analytical derivations. 

{\bf Problem of False Positives:} Terms `false positive' and `false negative' are used in general to highlight the limitations of `observation - inference' based methodologies. \emph{False positive} refers to a situation where observations and consequent inferences on a model result into a claim of the presence of certain property in the model which actually does not exist, while \emph{false negative} is used to refer the situation where specific observations do not yield required support for the presence of certain property, which is actually present in the model. False negatives are usually the result of incomplete observations while false positives indicate arbitrariness in the observation/inference process.  Like any other generic specification framework, current framework also suffers from the weakness of administering false positives. False negatives are also possible, whereby an observation process is defined such that it does not infer evolution, even though there might actually be evolution present in the model. The case of false negatives, however will not concern us since our focus is to establish the presence of evolution in a given AES model and not whether it is absent with respect to certain observations. On the other hand, the problem of false positives stems due to the fact that the framework permits arbitrariness in the definition of entities and their causal relationships. In case of causal relationships, they are defined in the framework as observation dependent and might not be consistent with the underlying micro-level dynamics of the model (Section~\ref{reproformalized}). This arbitrariness might give rise to false claims on the presence of evolutionary components in the model though there might be none actually.

\subsection{Further work}
\label{sec:further-work}

Framework can be further extended in several interesting directions, including the following: 
\begin{itemize}
\item We limit our attention to only those observations having evolutionary significance, though observations can also be made upon the model to establish other macro level  emergent properties including metabolic processes~\cite{ac:BFF92}, complexity~\cite{ac:adami00}, self organization~\cite{ac:Kau93}, autonomy and autopoisis~\cite{Zeleny81}, an impromptu response to the environmental inputs, compartmentalization, adaptability, and selectivity~\cite{Daniel02}.
\item {\it Experimentation} is yet another very important aspect of AES studies and there exists ample scope for the \textit{algorithmic or programmed  experimentation} with the AES models. Hence an interesting direction where current framework can be extended is by considering the {\it experimental-observational processes} with algorithmic experimentation and consequent discovery of life-like evolutionary behavior in AES studies. 

For example, Lohn et al. developed an automatic discovery system for the state-transition rules of CA such that a given structure could replicate itself in the CA space~\cite{Lohn97automaticdiscovery}. With the given initial structure, the state-transition rules would be evolved using the genetic algorithms~\cite{GA97} to make the structure self-replicate.
\item {\it Tighter Bounds.} Associated computational complexity analysis can be further refined and strengthened by considering classes of models for which most of the parameters have precise bounds compared to the generic analysis presented in this paper.    
\item Capturing the essence of \emph{strong emergence} by considering several observational processes at different organizational levels of the model.
\item Studying overlapping evolutionary processes - examples from real life include co-evolution, and sexual selection versus environmental selection.
\item Further constraints to overcome the problem of false positives by limiting as to what could be claimed as observed.
\end{itemize}

\bibliographystyle{plain}
\bibliography{alife}



\end{document}